%% file: main.tex
\documentclass[conference]{IEEEtran}
\usepackage{times}
\usepackage{graphicx}

\usepackage[numbers]{natbib}
\usepackage{multicol}
\usepackage[bookmarks=true]{hyperref}

\usepackage{comment}
\usepackage{xspace}
\let\labelindent\relax 
\usepackage{enumitem}
\usepackage{amsmath,amssymb,amsfonts,amsthm}
\usepackage{algorithm,algorithmicx}
\usepackage[noend]{algpseudocode}
\algrenewcommand\algorithmicindent{0.5em} 
\usepackage{parskip}
\usepackage{subcaption}
\usepackage[dvipsnames]{xcolor}
\usepackage{balance}

\newtheorem{lemma}{Lemma}

\DeclareMathOperator*{\argmin}{arg\,min}

\graphicspath{
    {./figs/}
}

\hypersetup{letterpaper,bookmarksopen,bookmarksnumbered,
pdfpagemode=UseOutlines,
colorlinks=true,
linkcolor=blue,
anchorcolor=blue,
citecolor=blue,
filecolor=blue,
menucolor=blue,
urlcolor=blue
}

\input{macros.tex}
\IEEEoverridecommandlockouts
\begin{document}

\title{Provably Constant-time Planning and Replanning for Real-time Grasping Objects off a Conveyor Belt}


\author{ Fahad Islam$^{1}$,
  Oren Salzman$^{2}$,
  Aditya Agarwal$^{1}$,
  Maxim Likhachev$^{1}$\\
    \authorblockA{
        $^{1}$The Robotics Institute, Carnegie Mellon University}
    \authorblockA{
        $^{2}$Technion-Israel Institute of Technology
        }
    \thanks{This work was supported by the ONR grant N00014-18-1-2775 and the ARL grant W911NF-18-2-0218 as part of the A2I2 program.}
}



%

\maketitle


\begin{abstract}
In warehouse and manufacturing environments, manipulation platforms are frequently deployed at conveyor belts to perform pick and place tasks. Because objects on the conveyor belts are moving, robots have limited time to pick them up. This brings the requirement for fast and reliable motion planners that could provide provable real-time planning guarantees, which the existing algorithms do not provide. Besides the planning efficiency, the success of manipulation tasks relies heavily on the accuracy of the perception system which is often noisy, especially if the target objects are perceived from a distance. For fast moving conveyor belts, the robot cannot wait for a perfect estimate before it starts executing its motion. In order to be able to reach the object in time it must start moving early on (relying on the initial noisy estimates) and adjust its motion on-the-fly in response to the pose updates from perception. We propose an approach that meets these requirements by providing provable constant-time planning and replanning guarantees. We present it, give its analytical properties and show experimental analysis in simulation and on a real robot.
\end{abstract}

\IEEEpeerreviewmaketitle

\section{Introduction}

Conveyor belts are widely used in automated distribution, warehousing, as well as for manufacturing and production facilities. In the modern times robotic manipulators are being deployed extensively at the conveyor belts for automation and faster operations~\cite{zhang2018gilbreth}. In order to maintain a high-distribution throughput, manipulators must pick up moving objects without having to stop the conveyor for every grasp. In this work, we consider the problem of motion planning for grasping moving objects off a conveyor. An object in motion imposes a requirement that it should be picked up in a short window of time. The motion planner for the arm, therefore, must compute a path within a bounded time frame to be able to successfully perform this task.

Manipulation relies on high quality detection and localization of moving objects. When the object first enters the robot's field of view, the initial perception estimates of the object's pose are often inaccurate. Consider the example of an object (sugar box) moving along the conveyor towards the robot in Fig.~\ref{fig:intro_pic}, shown through an image sequence as captured by the robot's Kinect camera in Fig.~\ref{fig:pose_sequence}. 
The plot in Fig.~\ref{fig:pose_sequence} shows the variation of the error between the filtered input point cloud and a point cloud computed from the predicted pose from our ICP based perception strategy as the object gets closer to the camera. We observe that the error decreases as the object moves closer, indicating that the point clouds overlap more closely due to more accurate pose estimates closer to the camera.

However, if the robot waits too long to get an accurate estimate of the object pose, the delay in starting plan execution could cause the robot to miss the object. The likelihood of this occurring increases proportionately with the speed of the conveyor. Therefore, the robot should start executing a plan computed for the initial pose and as it gets better estimates, it should repeatedly replan for the new goals. However, for every replanning query, the time window for the pickup shrinks. This makes the planner's job difficult to support real-time planning.

\begin{figure}[t]
    \centering
    \includegraphics[trim=0 50 0 100, clip, width=0.32\textwidth]{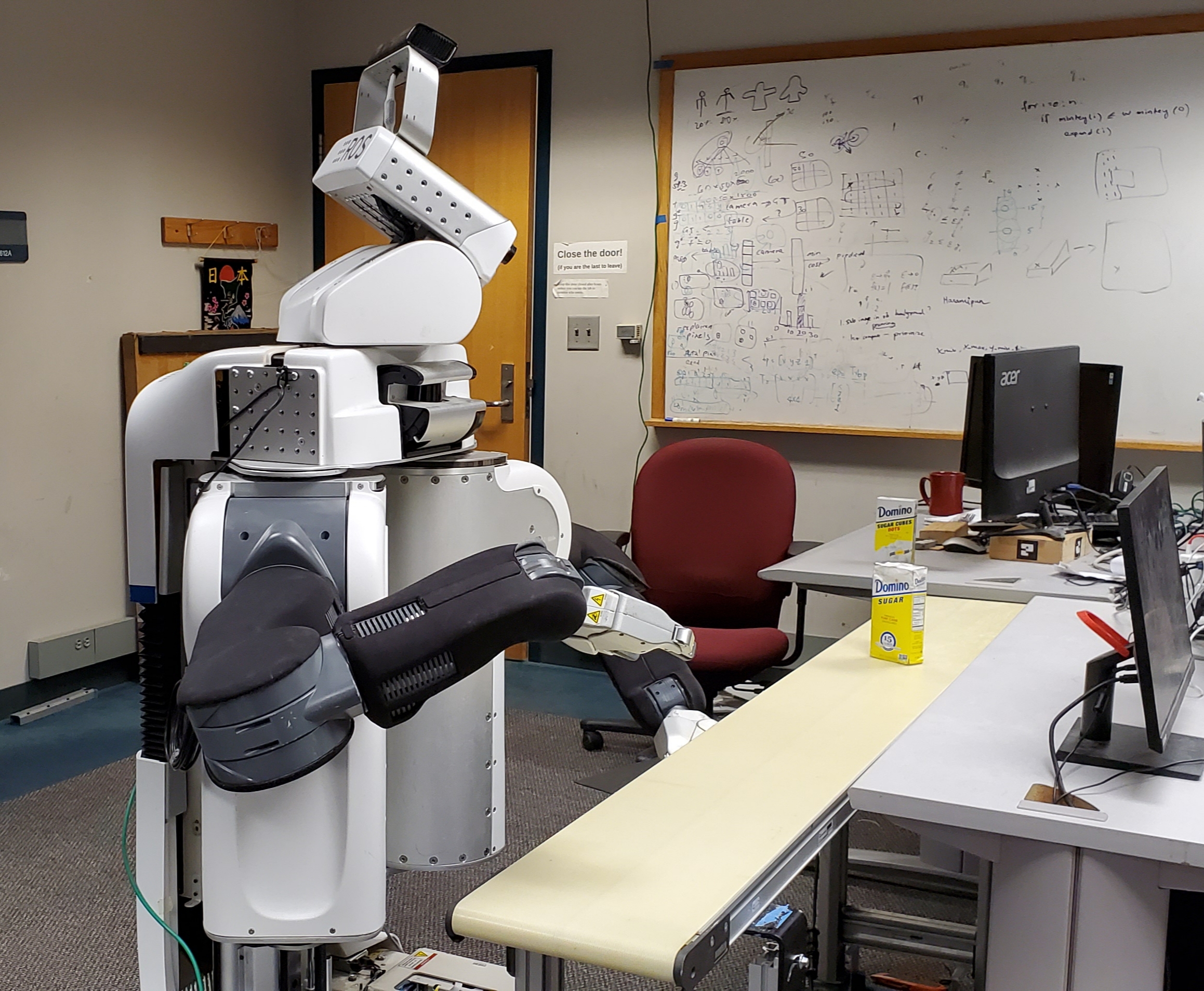}
    \caption{
    \CaptionTextSize
    A scene demonstrating the PR2 robot picking up a moving object (sugar box) off a conveyor belt.}
    \label{fig:intro_pic}
    \vspace{-4mm}
\end{figure}

Furthermore, the planning problem is challenging because the motion planner has to account for the dynamic object and thus plan with time as one of the planning dimension. It should generate a valid trajectory that avoids collision with the environment around it and also with the target object to ensure that it does not damage or topple it during the grasp. Avoiding collisions with the object requires precise geometric collision checking between the object geometry and the geometry of the manipulator. The resulting complexity of the planning problem makes it infeasible to plan online for this task.

Motivated by these challenges, we propose an algorithm that leverages offline preprocessing to provide bounds on the planning time when the planner is invoked online. Our key insight is that in our domain the manipulation task is highly repetitive. Even for different object poses, the computed paths are quite similar and can be efficiently reused to speed up online planning. Based on this insight, we derive a method that precomputes a representative set of paths with some auxiliary datastructures offline and uses them online in a way that provides \emph{constant-time planning guarantee}. Here, we assume that the geometric models of the target objects are known apriori. To the best of our knowledge, our approach is the first to provide constant-time planning guarantee on generating motions all the way to the goal for a dynamic environment.

We experimentally show that constant-time planning and replanning capability is necessary for a successful conveyor pickup task. Specifically if we only perform one-time planning, (namely, either following the plan for the initial potentially inaccurate pose estimate or from a delayed but accurate pose estimate) the robot frequently fails to pick the object.
\begin{figure}[t]
    \centering
    \begin{subfigure}{.225\textwidth}
        \includegraphics[height=2.9cm]{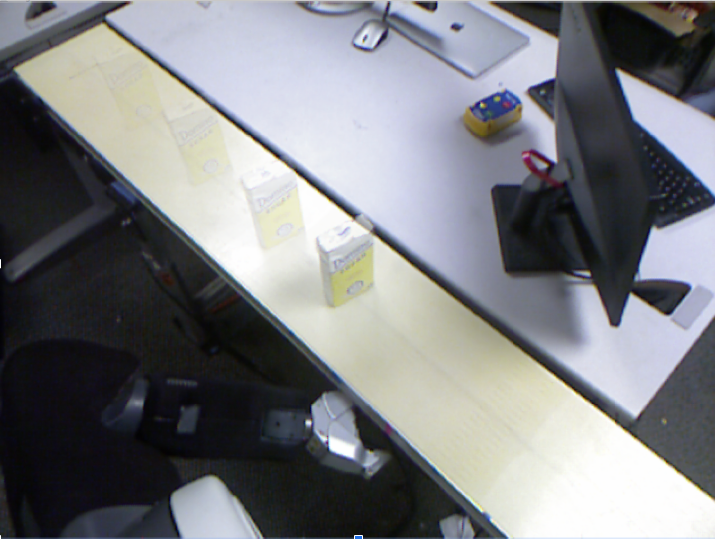}
        \caption{}
        \label{fig:obj1}
    \end{subfigure}
    \begin{subfigure}{0.225\textwidth}
        \includegraphics[height=2.9cm]{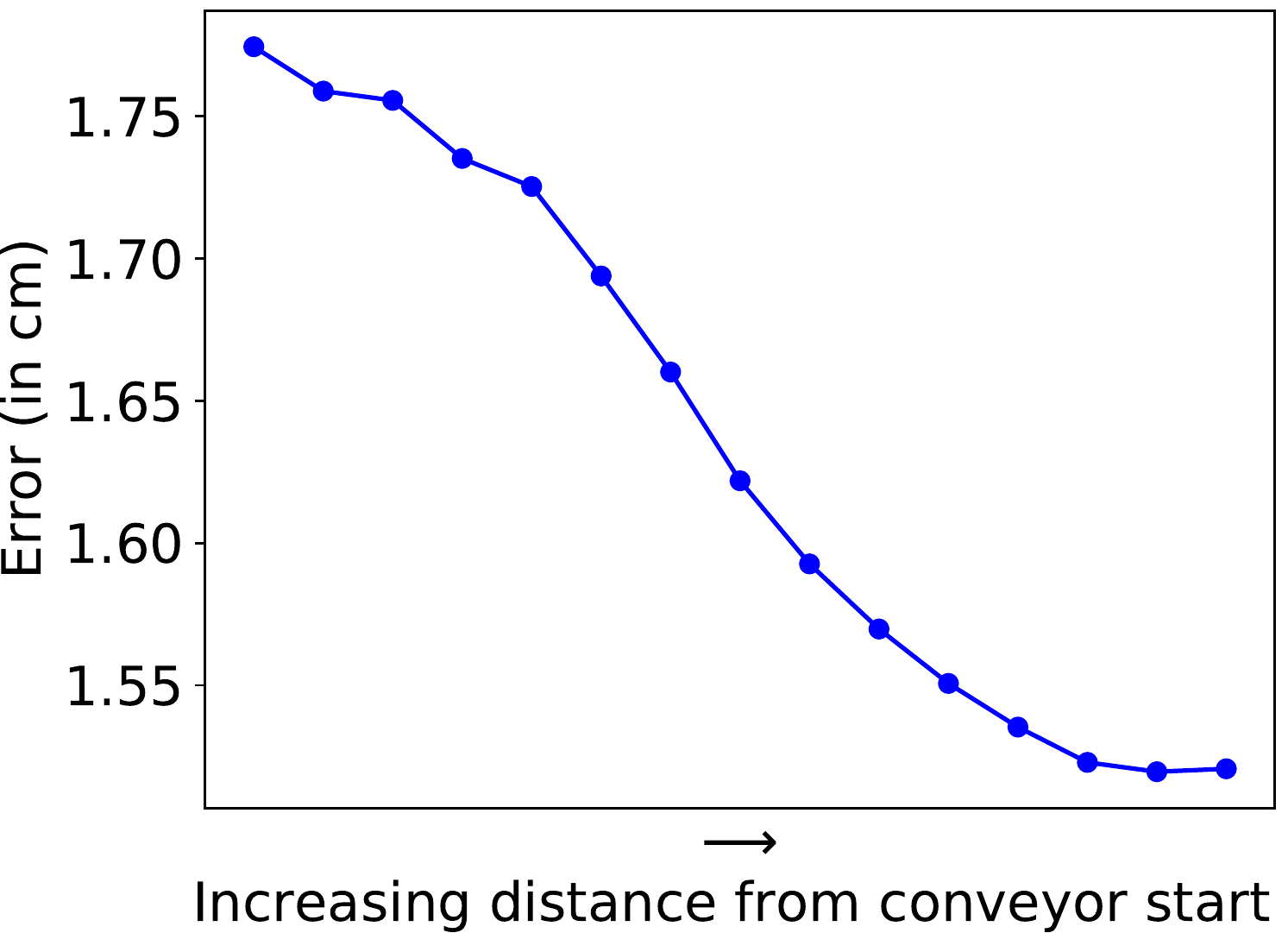}
        \caption{}
        \label{fig:obj2}
    \end{subfigure}
    \caption{
    \CaptionTextSize
    (\subref{fig:obj1})~Depiction of an object moving along a conveyor towards the robot.
    (\subref{fig:obj2})~Pose error as a function of the distance from the conveyor's start. Specifically we use ADD-S error \cite{add_metric}.
    }
    \label{fig:pose_sequence}
    \vspace{-4mm}
\end{figure}

\section{Related work}
\subsection{Motion planning for conveyor pickup task}
Existing work on picking moving objects has focused on different aspects of the problem ranging from closed-loop controls to object perception and pose estimation, motion planning and others~\cite{allen1993automated, han2019toward, stogl2017tracking, zhang2018gilbreth}. 
Here, we focus on motion-planning related work. Time-configuration space representation was introduced to avoid moving obstacles~\cite{fraichard1993dynamic,cefalo2013task,yang2018planning}. Specifically in~\cite{yang2018planning}, a bidirectional sampling-based method with a time-configuration space representation was used to plan motions in dynamic environments to pickup moving objects. While their method showed real-time performance in complex tasks, it used fully-specified goals; namely knowing the time at which the object should be picked, which weakens the completeness guarantee. Furthermore their method is probablistically complete and therefore, does not offer constant-time behavior.
Graph-search based approaches have also been used for the motion-planning problem~\cite{menon2014motion, cowley2013perception}. The former uses a kinodynamic motion planner to smoothly pick up moving objects i.e., without an impactful contact. A heuristic search-based motion planner that plans with dynamics and could generate optimal trajectories with respect to the time of execution was used. While this planner provides strong optimality guarantees, it is not real-time and thus cannot be used online.
The latter work demonstrated online real-time planning capability. The approach plans to a pregrasp pose with pure kinematic planning and relies on Cartesian-space controllers to perform the pick up. The usage of the Cartesian controller limits the types of objects that the robot can grasp.

\subsection{Preprocessing-based planning}
Preprocessing-based motion planners often prove beneficial for real-time planning. They analyze the configuration space offline to generate some auxiliary information that can be used online to speed up planning. 
Probably the best-known example is the Probablistic Roadmap Method (PRM)~\cite{kavraki1996probabilistic} which precomputes a roadmap that can answer any query by connecting the start and goal configurations to the roadmap and then searching the roadmap. PRMs are fast to query yet they do not provide constant-time guarantees.
Moreover, in our case, to account for a moving object, they would require edge re-evaluation which is often computationally expensive.

A provably constant-time planner was recently proposed in~\cite{ISL19}. Given a start state and a goal region, it precomputes a compressed set of paths that can be utilized online to plan to any goal within the goal region in bounded time. As we will see, our approach while algorithmically different, draws some inspiration from this work.
Both of the above two methods (\cite{ISL19,kavraki1996probabilistic}) are mainly targetting pure kinematic planning and thus they cannot be used for the conveyor-planning problem which is dynamic in nature.

Another family of preprocessing-based planners utilizes previous experiences to speed up the search~\cite{BAG12,CSMOC15,PCCL12}. Experience graphs~\cite{PCCL12}, provide speed up in planning times for repetitive tasks by trying to reuse previous experiences. These methods are also augmented with sparsification techniques (see e.g.,~\cite{DB14,SSAH14}) to reduce the memory footprint of the algorithm.
Unfortunately, none of the mentioned algorithms provide fixed planning-time guarantees that we strive for in our application.

\subsection{Online replanning and real time planning}
The conveyor-planning problem can be modelled as a Moving Target Search problem (MTS) which is a widely-studied topic in the graph search-based planning literature~\cite{ishida1991moving,ishida1995moving,koenig2007speeding,sun2010moving}. 
These approaches interleave planning and execution incrementally and update the heuristic values of the state space to improve the distance estimates to the moving target. Unfortunately, in high-dimensional planning problems, this process is computationally expensive which is why these approaches are typically used for two-dimensional grid problem such as those encountered in video games. More generally, real-time planning is widely considered in the search community (see, e.g.,~\cite{KL06,KS09,korf1990real}).
However, as mentioned, these works are typically applicable to low-dimensional search spaces.
%

\section{Problem definition}
Our system is comprised of 
a robot manipulator~$\calR$,
a conveyor belt~$\calB$ moving at some known velocity,
a set of known objects~$\calO$ that need to be grasped and 
a perception system~$\calP$ that is able to estimate the type of object and its location on~$\calB$.

Given a pose $g$ of an object $o \in \calO$, our task is to plan the motion of $\calR$ such that it grasps~$o$ from~$\calB$ at some future time.
Unfortunately, the perception system $\calP$ may give inaccurate object poses.
Thus, the pose $g$ will be updated by~$\calP$ as $\calR$ is executing its motion. 
To allow for $\calR$ to move towards the updated pose in real time, we introduce the additional requirement that planning should be done within a user-specified time bound~\Tbound.
For ease of exposition, when we say that we plan to a pose $g$ of $o$ that is given by $\calP$, 
we mean that we plan the motion of $\calR$ such that it will be able to pick~$o$ from~$\calB$ at some future time. 
This is explained in detail in Sec.~\ref{sec:eval} and in Fig.~\ref{fig:pe}.

We denote by $\Gfull$ the discrete set of initial object poses on $\calB$ that $\calP$ can perceive.
Finally, we assume that $\calR$ has an initial state \Shome corresponding to the time $t=0$ from which it starts planning to grasp any object.

Roughly speaking, the objective, following the set of assumptions we will shortly state, is to enable planning and replanning to any goal pose $ g \in \Gfull$ in bounded time~\Tbound regardless of $\calR$'s current state.
To formalize this idea, let us introduce the notion of \emph{reachable} and \emph{covered} states:

\vspace{2mm}
\begin{definition}
    A goal pose $g \in \Gfull$ is said to be \emph{reachable} from a state $s$ if there exists a path from $s$ to $g$ and it can be computed in finite time.
\end{definition}

\vspace{2mm}
\begin{definition}
    A reachable pose $g \in \Gfull$ is said to be \emph{covered} by a state $s$ if 
    the planner can find a path from $s$ to $g$ within time~\Tbound.
\end{definition}

Thus, we wish to build a system such that 
for any state $s$ that the system can be in 
and every reachable goal pose $g \in \Gfull$ from $s$ updated by~$\calP$,
$g$ is covered by $s$.

We are now ready to state the assumptions for which we can solve the problem defined.

\begin{enumerate}[label={\textbf{A\arabic*}},leftmargin=0.75cm]
    
\ignore{
    \item \label{assum:2} Given a path 
    $\Pi = \{s_0, \ldots, s_k \}$ 
    s.t. $s_0 = \Shome$ and $s_k \in \Gfull$, 
    we have that $G^{\rm reach}(s_{i+1}) \subset G^{\rm reach}(s_{i})$.
    Namely, the reachable set of goals for a state on the path is a subset of the reachable set of every other state on that path that exists before it.
    \os{Oren - Didn't we agree that this is not an assumption but a property?}
}

    \item \label{assum:4} There exists a replan cutoff time $t=\Trc$, after which the planner does not replan and continues to execute the last planned path.

    \item \label{assum:5} For any time $0 \leq t \leq \Trc$, the environment is static. Namely, objects on~$\calB$ cannot collide with $\calR$ during that time.

    \item \label{assum:3} The pose estimation error of $\calP$ is bounded by a distance $\varepsilon_{\calP}$.
        

\end{enumerate}

Assumptions~\ref{assum:4}-\ref{assum:5} enforce a requirement that $\calP$ must converge on an accurate estimate $g$ before \Trc, and until \Trc, $o$ is guaranteed to be at a safe distance from $\calR$.
Assumption~\ref{assum:3} bounds the maximum error of $\calP$ and is explained in detail in Sec.~\ref{sec:eval} and in Fig.~\ref{fig:pe}


\section{Algorithmic framework}
\label{subsec:strawman}
Our approach for constant-time planning relies on a \emph{preprocessing} stage that allows to efficiently compute paths in a \emph{query} stage to any goal (under Assumptions~\ref{assum:4}-\ref{assum:3}). 
Before we describe our approach, we start by describing a \naive method that solves the aforementioned problem but requires a prohibitive amount of memory.
This can be seen as a warmup before describing our algorithm which exhibits the same traits but does so in a memory-efficient manner.

\subsection{Straw man approach}
\begin{figure}[t]
    \centering
    \hspace{3mm}
    \begin{subfigure}{.225\textwidth}
        \includegraphics[width=0.9\textwidth]{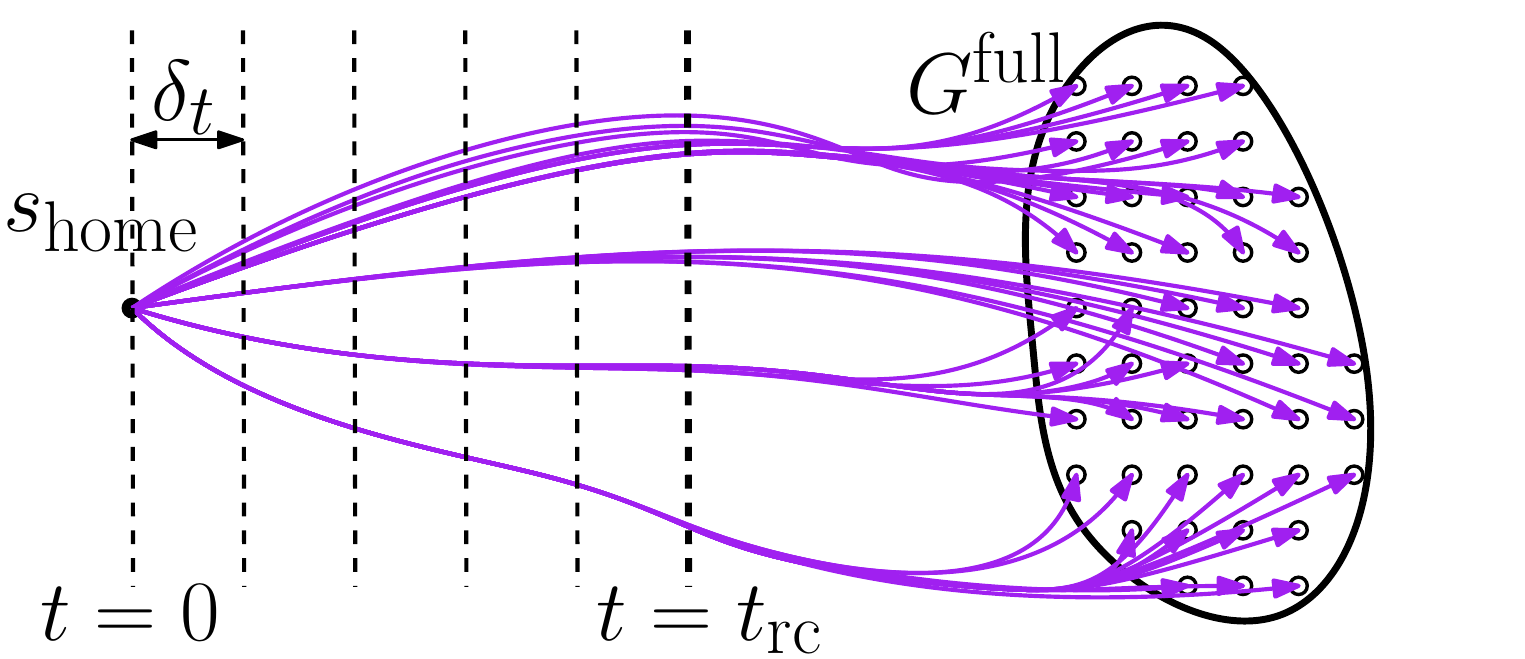}
        \caption{}
        \label{fig:naive1}
    \end{subfigure}
    \begin{subfigure}{0.225\textwidth}
        \includegraphics[width=0.9\textwidth]{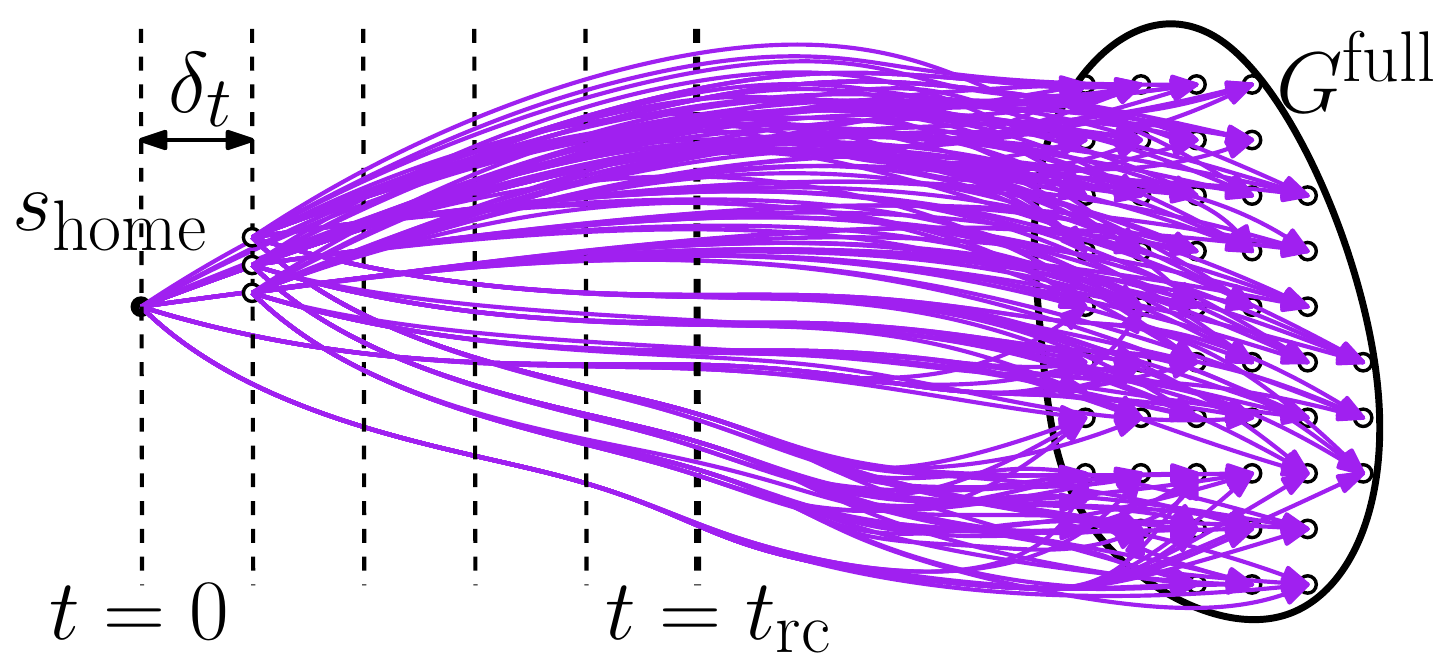}
        \caption{}
        \label{fig:naive2}
    \end{subfigure}
    \caption{
    \CaptionTextSize
    The figures show paths discretized from timesteps $t_0$ to $\Trc$ with steps of size $\delta_t$.
    (\subref{fig:naive1})~At $t=0$, the algorithm computes $n_{\rm goal}$ paths, that is from \Shome to every $g \in \Gfull$.
    (\subref{fig:naive2})~At $t=\delta_t$, the algorithm computes $n_{\rm goal}^2$ paths, that is from all $n_{\rm goal}$ replanable states at $t=\delta_t$ to every $g \in \Gfull$ (here we only show paths from three states).
    Thus, the number of paths increases exponentially at every timestep.
    }
    \label{fig:naive}
    \vspace{-4mm}
\end{figure}

We first compute from \Shome a path $\pi_g$ to every reachable $g \in \Gfull$. 
These paths can be stored in a lookup (hash) table which can be queried in constant time (assuming perfect hashing~\cite{czech1997perfect}). Thus, all goals are covered by \Shome and this allows us to start executing a path once $\calP$ gives its initial pose estimate.
However, we need to account for pose update while executing~$\pi_g$. 
Following \ref{assum:4} and \ref{assum:5}, this only needs to be done up until time~\Trc.
Thus, we discretize each path uniformly with resolution $\delta_t$.
%
We call all states that are less than \Trc time from \Shome \emph{replanable states}.

Next, for every replanable state along each path $\pi_g$, we compute a new path to all goals. 
%
This will ensure that all goals are covered by all replanable states. Namely, it will allow to immediately start executing a new path once the goal location is updated by~$\calP$.
Unfortunately,~$\calP$ may update the goal location more than once. Thus, this process needs to be performed recursively for the new paths as well.

The outcome of the preprocessing stage is a set of precomputed collision-free paths starting at states that are at most $\Trc$ from \Shome and end at goal states.
The paths are stored in a lookup table $\calM: S \times \Gfull \rightarrow \{ \pi_1, \pi_2, \ldots \}$ that can be queried in $O(1) (< \Tbound)$ time to find a path from any given $s \in S$ to $g \in \Gfull$.

In the query stage we obtain an estimation $g_1$ of the goal pose by $\calP$. 
The algorithm then retrieves the path~$\pi_1(\Shome,g_1)$ (from~\Shome to~$g_1$) from $\calM$ and the robot starts executing~$\pi_1(\Shome,g_1)$.
For every new estimation $g_i$ of the goal pose obtained from~$\calP$  while the system is executing path $\pi_{i-1}(s_{i-1},g_{i-1})$, the algorithm retrieves from $\calM$ the path $\pi_i(s_i,g_i)$ from the first state~$s_i$ along $\pi_{i-1}(s_{i-1},g_{i-1})$ that is least $\Tbound$ away from~$s_{i-1}$. The robot~$\calR$ will then start executing~$\pi_i(s_i,g_i)$ once it reaches~$s_i$.

Clearly, every state is covered by this brute-force approach, however it requires a massive amount of memory.
Let $n_{\rm goal} = \vert \Gfull \vert$ be the number of goals and
$\ell$ be the number of states between \Shome and the state that is \Trc time away.
This approach requires precomputing and storing $O(n_{\rm goal}^\ell)$ paths which is clearly infeasible (see Fig.~\ref{fig:naive}).
In the next sections, we show how we can dramatically reduce the memory footprint of the approach without compromising on the system's capabilities.

\subsection{Algorithmic approach}
While the straw man algorithm presented  allows for planning to any goal pose $ g \in \Gfull$ in bounded time~\Tbound, its memory footprint is prohibitively large.
We suggest to reduce the memory footprint by building on the observation that many paths to close-by goals traverse very similar parts of the configurations space.


The key idea of our approach is that instead of computing (and storing) paths to all reachable goals in \Gfull, we compute a relatively small subset of so-called ``root paths" that can be reused in such a way that we can still cover \Gfull fully. Namely, at query time, we can reuse these paths to plan to any $g\in \Gfull$ within \Tbound. The idea is illustrated in Fig.~\ref{fig:crp}.

First, we compute a set of root paths $\{\Pi_1, \ldots, \Pi_k \}$ from \Shome to cover \Gfull by \Shome (here we will have that $k \ll n_{\rm goal})$ 
Next, the algorithm recursively computes for all replanabale states along these root paths, additional root paths so that their reachable goals are also covered.
During this process, additional root paths are computed only when the already existing set of root paths does not provide enough guidance to the search to cover \Gfull i.e to be able to compute a path to any $g \in \Gfull$ within \Tbound.
%
%
The remainder of this section formalizes these ideas.

\subsection{Algorithmic building blocks}
We start by introducing the algorithmic building blocks that we use.
Specifically, we start by describing the motion planner that is used to compute the root paths 
and then continue to describe how they can be used as \emph{experiences} to efficiently compute paths to other goals.
\subsubsection{Motion planner}
We use a heuristic search-based planning approach with motion primitives (see, e.g,~\cite{CCL10,CSCL11,LF09})
as it allows for deterministic planning time which is key in our domain.
Moreover, such planners can easily handle under-defined goals as we have in our setting---we define a goal as a grasp pose for the goal object while the planning dimension includes the DoFs of the robot as well as the time dimension.

\textbf{State space and graph construction.}
We define a state $s$ as a pair $(q,t)$ where $q = (\theta_1, ..., \theta_n)$ is a configuration represented by the joint angles for an $n$-DOF robot arm (in our setting $n=7$) and $t$ is the time associated with $q$.
Given a state $s$ we define two types of motion primitives which are short kinodynamically feasible motions that~$\calR$ can execute.
%

The first type of motion primitives are predefined primitives. These are small individual joint movements in either direction as well as \emph{wait} actions.
For each motion primitive, we compute its duration by using a nominal constant velocity profile for the  joint that is moved.
%

The second type of primitives are dynamic primitives. They are generated by the search only at the states that represent the arm configurations where the end effector is close to the object. These primitives correspond to the actual grasping of the object while it is moving.
The dynamic primitives are generated by using a Jacobian pseudo inverse-based control law similar to what~\cite{menon2014motion} used. 
The velocity of the end effector is computed such that the end-effector minimizes the distance to the grasp pose. Once the gripper encloses the object, it moves along with the object until the gripper is closed.

\textbf{Motion planner.}
The states and the transitions implicitly define a graph $\calG = (S,E)$ where $S$ is the set of all states and~$E$ is the set of all transitions defined by the motion primitives. We use Weighted A* (wA*)~\cite{pohl1970heuristic} to find a path in $\calG$ from a given state~$s$ to a goal $g$. 
wA* is a suboptimal heuristic search algorithm that allows a tradeoff between optimality and greediness by inflating the heuristic function by a given weight~$w$. 
The search is guided by an efficient and fast-to-compute heuristic function which in our case has two components.
The first component drives the search to intercept the object at the right time and 
the second component guides the search to correct the orientation of the end effector as it approaches the object. 
Mathematically, our heuristic function is given by
$$
 h(s,g) = \max (\lambda \cdot t(s,g), \textsc{AngleDiff}(s,g)).
$$
Here, $t(s,g)$ is the expected time to intercept the object which can be analytically computed from the velocities and positions of the target object and the end-effector and \textsc{AngleDiff}($s,g$) gives the magnitude of angular difference between the end-effector's current pose and target pose. The coefficient $\lambda$ is used as a weighting factor.

\subsubsection{Planning with Experience Reuse}
We now show how previously-computed paths which we named as root paths can be reused as experiences in our framework. Given a heuristic function $h$ we define for a root path $\Pi$ and a goal $g \in \Gfull$ the \emph{shortcut} state $\Ssc (\Pi,g)$ as the  state that is closest to~$g$ with respect~$h$.
Namely,
$$
\Ssc (\Pi,g) := \argmin\limits_{s_i \in \Pi} h(s_i, g).
$$
Now, when searching for a path to a goal $g \in \Gfull$ using root path $\Pi$ as an experience, we add $\Ssc (\Pi,g)$ as a successor for any state along~$\Pi$
(subject to the constraint that the path along $\Pi$ to \Ssc is collision free). In this manner we reuse previous experience to quickly reach a state close to the $g$.

\subsection{Algorithmic details}
We are finally ready to describe our algorithm describing first the preprocessing phase and then  the query phase.

\subsubsection{Preprocessing}
\begin{figure*}[t]
    \centering
    \begin{subfigure}{.225\textwidth}
        \includegraphics[width=\textwidth]{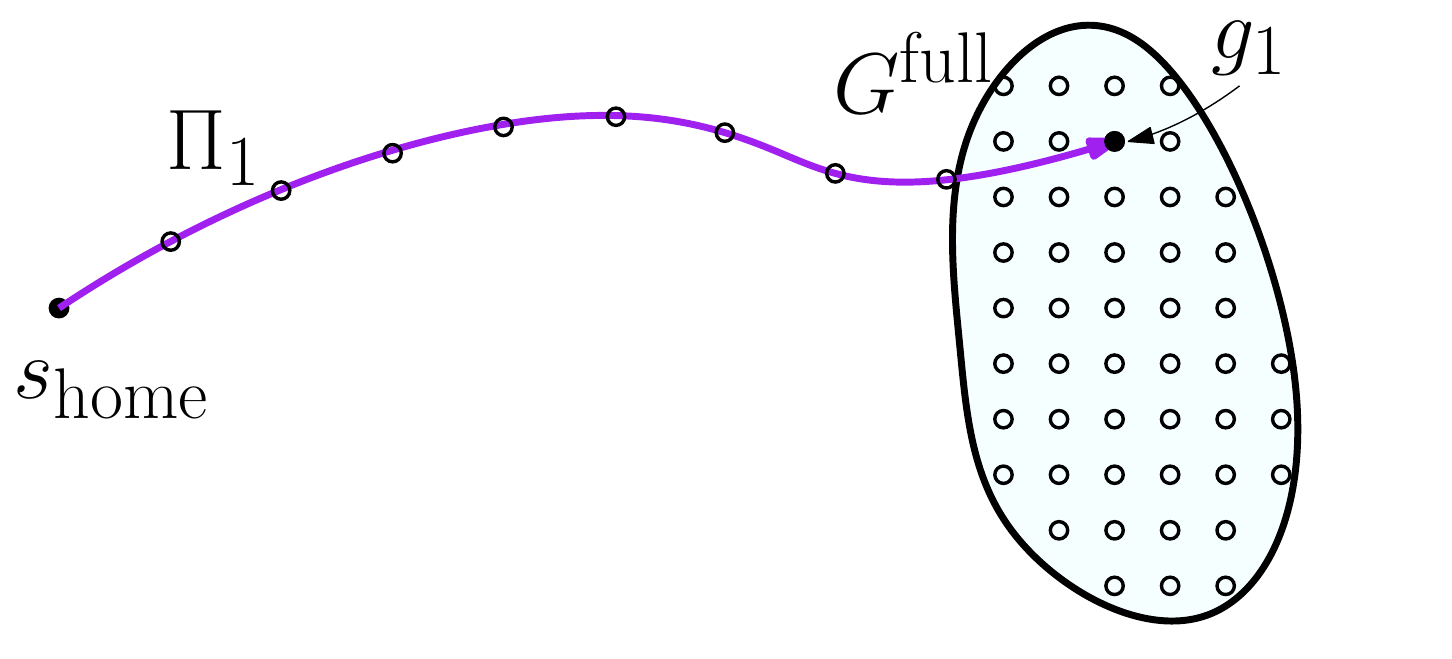}
        \caption{}
        \label{fig:crp1}
    \end{subfigure}
    \hspace{4mm}
    \begin{subfigure}{0.225\textwidth}
        \includegraphics[width=\textwidth]{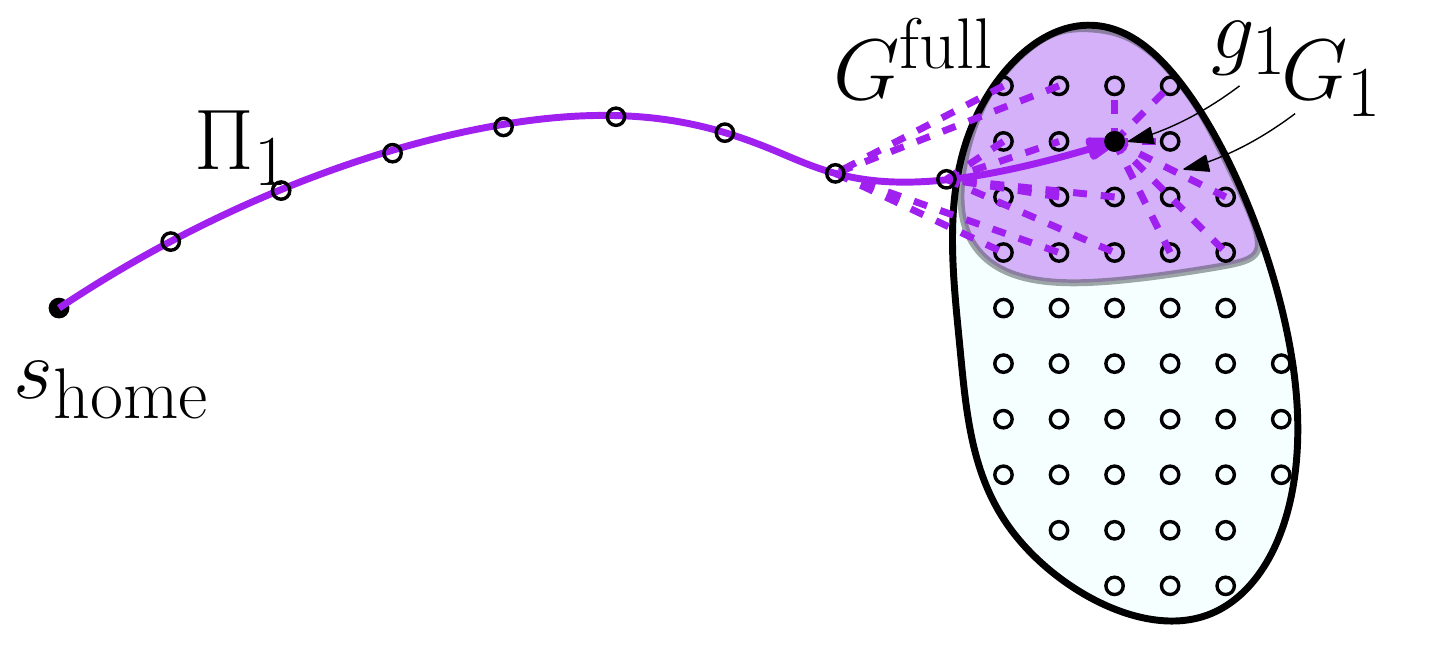}
        \caption{}
        \label{fig:crp2}
    \end{subfigure}
    \hspace{4mm}
    \begin{subfigure}{0.225\textwidth}
        \includegraphics[width=\textwidth]{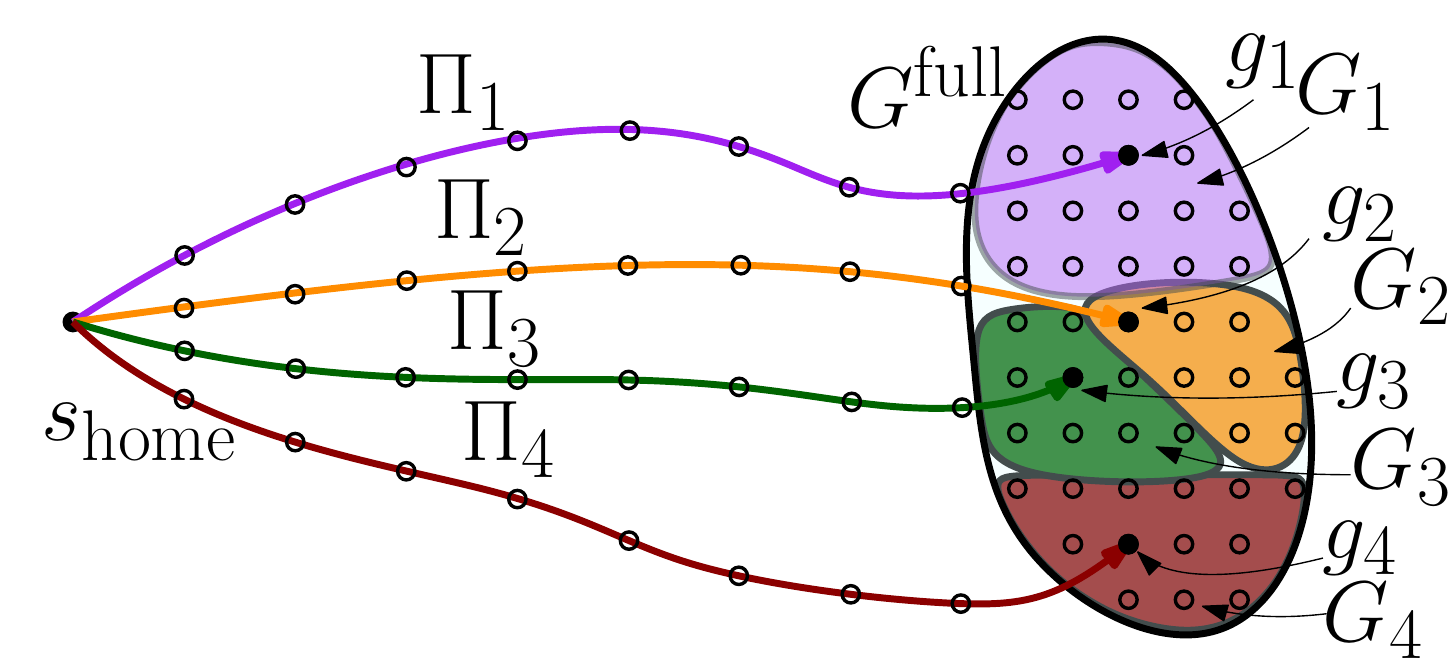}
        \caption{}
        \label{fig:crp3}
    \end{subfigure}
    \caption{\CaptionTextSize
    First step of the preprocessing stage.
    (\subref{fig:crp1})~A goal $g_1$ is sampled and the root path $\Pi_1$ is computed between \Shome and $g_1$.
    (\subref{fig:crp2})~The set $G_1 \subset \Gfull$ of all states that can use $\Pi_1$ as an experience is computed and associated with $\Pi_1$.
    (\subref{fig:crp3})~The goal region covered by four root paths from \Shome after the first step of the preprocessing stage terminates.
    }
    \label{fig:crp}
\end{figure*}

Our preprocessing stage starts by sampling a goal~$g_1 \in \Gfull$ and computing a root path~$\Pi_1$ from~$\Shome$ to~$g_1$. We then associate with~$\Pi_1$ the set of goals~$G_1 \subset \Gfull$ such that~$\Pi_1$ can be used as an experience in reaching any $g_j \in G_1$ within~\Tbound\footnote{In practice, to account for other query phase operations, such as hash table lookups etc., a slightly smaller time than \Tbound is provided to the experience-based planner, to ensure that the overall query time is bounded by \Tbound.}.
Thus, all goals in~$G_1$ are covered by~\Shome.
We then repeat this process but instead of sampling  a goal from \Gfull, we sample from $\Gfull \setminus G_1$, thereby removing covered goals from \Gfull in every iteration.
At the end of this step, we obtain a set of root paths. 
Each root path~$\Pi_i$ is associated with a goal set $G_i \subseteq \Gfull$ such that 
(i)~$\Pi_i$ can be used as an experience for planning to any $g_j \in G_i$ in~\Tbound and 
(ii)~$\bigcup_i G_i = \textsc{Reachable}(\Shome, \Gfull)$ (i.e all reachable goals for \Shome in \Gfull).
Alg.~\ref{alg:step1} details this step (when called with arguments ($\Shome,\Gfull$)). It also returns a set of unreachable goals that are left uncovered.
The process is illustrated in Fig.~\ref{fig:crp}.

\begin{algorithm}[t]
\caption{Plan Root Paths}
\label{alg:step1}
    \AlgFontSize
\begin{algorithmic}[1]
\Procedure{PlanRootPaths}{$s_{\textrm{start}}, \Guncov$}
\State $\Psi_{\Sstart} \leftarrow \emptyset$   \Comment{a list of pairs ($\Pi_i, G_i)$}
\State $\Guncov_{\Sstart} \leftarrow \emptyset$; \hspace{3mm}
       $i = 0$
\While{$\Guncov \neq \emptyset$}
        \Comment{until all reachable goals are covered}
    \State $g_i \leftarrow$\textsc{SampleGoal}($\Guncov$)
    \State $\Guncov \leftarrow \Guncov \setminus \{g_i\}$
    
    \If {$\Pi_i \leftarrow$ \textsc{PlanRootPath}($s_{\textrm{start}}, g_i$)} \Comment{planner succeeded}
        \State $G_i \leftarrow \{ g_i \}$   \Comment{goals reachable}
        \For {\textbf{each} $g_j \in \Guncov$}
            \If {$\pi_j \leftarrow$\textsc{PlanPathWithExperience}($s_{\textrm{start}},g_j,\Pi_i$)}
                \State $G_i \leftarrow G_i \cup \{g_j\}$
                \State $\Guncov \leftarrow \Guncov \setminus \{g_j\}$
            \EndIf
        \EndFor
        \State $\Psi_{\Sstart} \leftarrow \Psi_{\Sstart} \cup \{ (\Pi_i, G_i)\}$; \hspace{3mm}
        $i \leftarrow i + 1$
        
    \Else
        \State $\Guncov_{\Sstart} \leftarrow \Guncov_{\Sstart} \cup \{g_i\}$ \Comment{goals unreachable}
    \EndIf
\EndWhile
\State \textbf{return} $\Psi_{\Sstart}, \Guncov_{\Sstart}$
\EndProcedure
\end{algorithmic}
\end{algorithm}

%
\begin{algorithm}[t]
\caption{Preprocess}\label{alg:preprocess}
    \AlgFontSize
\begin{algorithmic}[1]
\Procedure{TryLatching}{$s,\Psi_{\Shome}\Guncov,\Gcov$}
        \For {\textbf{each} $(\Pi_i, G_i) \in \Psi_{\Shome}$}
        \label{alg:preprocess:line:latch1a}
            \If{\textsc{CanLatch}($s,\Pi_i$)}
                \State $\Guncov \leftarrow \Guncov \setminus G_i$
                \State $\Gcov \leftarrow \Gcov \cup G_i$
                \label{alg:preprocess:line:latch1b}
            \EndIf
        \EndFor
\State \textbf{return} $\Guncov, \Gcov$
\vspace{2mm}
\EndProcedure
\Procedure{Preprocess}{$\Sstart,\Guncov,\Gcov$}
\State $\Psi_{\Sstart}, G^{\textrm{uncov}}_{\Sstart} \leftarrow$ \textsc{PlanRootPaths}($\Sstart,\Guncov$)
{\color{blue}{ \State \textbf{if} $\Sstart = \Shome$ \textbf{then} $\Psi_{\Shome} = \Psi_{\Sstart}$}}


\State $G^{\textrm{cov}}_{\Sstart} \leftarrow 
    \Gcov \cup (\Guncov \setminus G^{\textrm{uncov}}_{\Sstart})$
\If{$t(s_{\textrm{start}}) \leq \Trc$}

\For {\textbf{each} $(\Pi_i, G_i) \in \Psi_{\Sstart}$} \label{loop1}
    \State $G_i^{\textrm{cov}} \leftarrow G_i$;
            \hspace{2mm}
           $G_i^{\textrm{uncov}} \leftarrow G^{\textrm{cov}}_{\Sstart} \setminus G_i$; \label{uncov_init}
            \hspace{2mm}

    \For {\textbf{each} $s \in \Pi_i$ ({from last to first})} \Comment{states up to $\Trc$} \label{loop2}
\textcolor{blue}{
        \State $\Guncov_i,\Gcov_i \leftarrow$ \textsc{TryLatching}($s,\Psi_{\Shome},\Guncov_i,\Gcov_i$)
        \If{$G_i^{\textrm{uncov}} = \emptyset$}
            \State \textbf{break}
        \EndIf
} 
        \State $G_i^{\textrm{uncov}},G_i^{\textrm{cov}} \leftarrow$ \textsc{Preprocess}($s,G_i^{\textrm{uncov}},G_i^{\textrm{cov}}$)    \label{recursion}
        \If{$G_i^{\textrm{uncov}} = \emptyset$} \label{terminate}
            \State \textbf{break}
        \EndIf
 
    \EndFor
\EndFor
\EndIf
\State \textbf{return} $G^{\textrm{uncov}}_{\Sstart}, G^{\textrm{cov}}_{\Sstart}$
\EndProcedure
\end{algorithmic}
\end{algorithm}

%


So far we explained the algorithm for one-time planning when the robot is at \Shome ($t = 0$); we now need to allow for efficient replanning for any state $s$ between $t = 0$ to $\Trc$. In order to do so, we iterate through all the states on these root paths and add additional root paths so that these states also cover their respective reachable goals. This has to be done recursively since newly-added paths generate new states which the robot may have to replan from. The complete process is detailed in Alg.~\ref{alg:preprocess}.
The \textsc{Preprocess} procedure takes in a state \Sstart, the goal region that it has to cover \Guncov and region that it already has covered \Gcov. Initially \textsc{Preprocess} is called with arguments ($\Shome, \Gfull, \emptyset$) and it runs recursively until no state is left with uncovered reachable goals.
%

At a high level, the algorithm iterates through each root path $\Pi_i$ (loop at line~\ref{loop1}) and for each state $s \in \Pi_i$ (loop at line~\ref{loop2}) the algorithm calls itself recursively (line~\ref{recursion}). The algorithm terminates when all states cover their reachable goals. The pseudocode in blue constitute an additional optimization step which we call ``latching" and is explained later in Sec.~\ref{latching}.

In order to minimize the required computation, the algorithm leverages two key observations:

\begin{enumerate}[label={\textbf{O\arabic*}},leftmargin=0.75cm]
    \item \label{obs:1}
    If a goal is not reachable from a state $s \in \Pi$, it is not reachable from all the states after it on $\Pi$.
    \item \label{obs:2}
    If a goal is covered by a state $s \in \Pi$, it is also covered by all states preceding it on $\Pi$.
\end{enumerate}

We use \ref{obs:1} to initialize the uncovered set for any state; instead of attempting to cover the entire \Gfull for each replanable state~$s$, the algorithm only attempts to cover the goals that could be reachable from $s$, thereby saving computation.
%
\ref{obs:2} is used by iterating backwards on each root path (loop at line~\ref{loop2}) and for each state on the root path only considering the goals that are left uncovered by the states that appear on the path after it.

Specifically,~\ref{obs:2} is used to have a single set of uncovered goals~$\Guncov_i$ for all states that appear on $\Pi_i$ instead of having individual sets for each state and the goals that each $s \in \Pi_i$ covers in every iteration of loop~\ref{loop2} are removed from~$\Guncov_i$.
\ref{obs:1} is used to initialize $\Guncov_i$ (in line~\ref{uncov_init}). Namely, it is initialized not by the entire \Gfull but by the set of goals covered by \Sstart. $G_i$ is excluded since it is already covered via $\Pi_i$. The iteration completes either when all goals in $\Guncov_i$ are covered (line~\ref{terminate}) or the loop backtracks to \Sstart.
The process is illustrated in Fig.~\ref{fig:pl_no_latching}

Thus, as the outcome of the preprocessing stage a map $\calM: S \times \Gfull \rightarrow \{ \Pi_1, \Pi_2, \ldots \}$ is constructed that can be looked up to find which root path can be used as an experience to plan to a goal $g$ from a state $s$ within \Tbound.

\begin{figure*}[t]
    \centering
    \begin{subfigure}{.225\textwidth}
        \includegraphics[width=\textwidth]{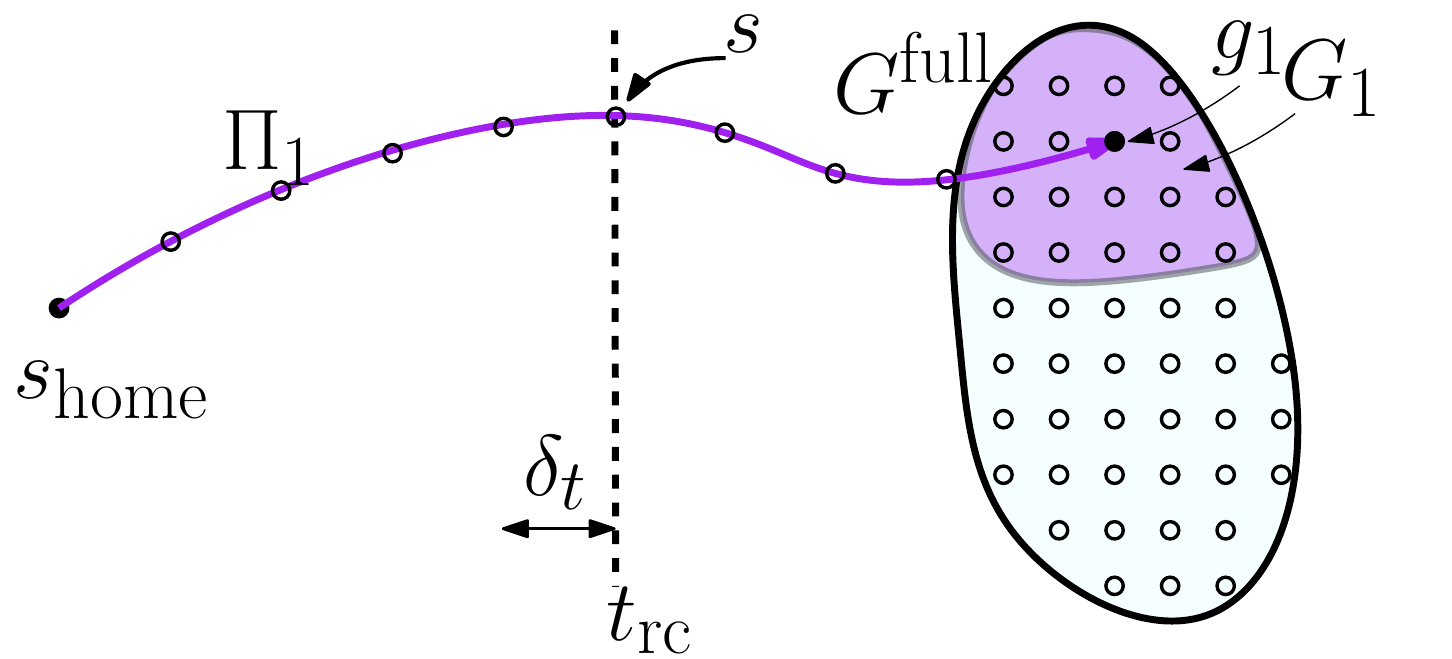}
        \caption{}
        \label{fig:pre1}
    \end{subfigure}
    \begin{subfigure}{0.225\textwidth}
        \includegraphics[width=\textwidth]{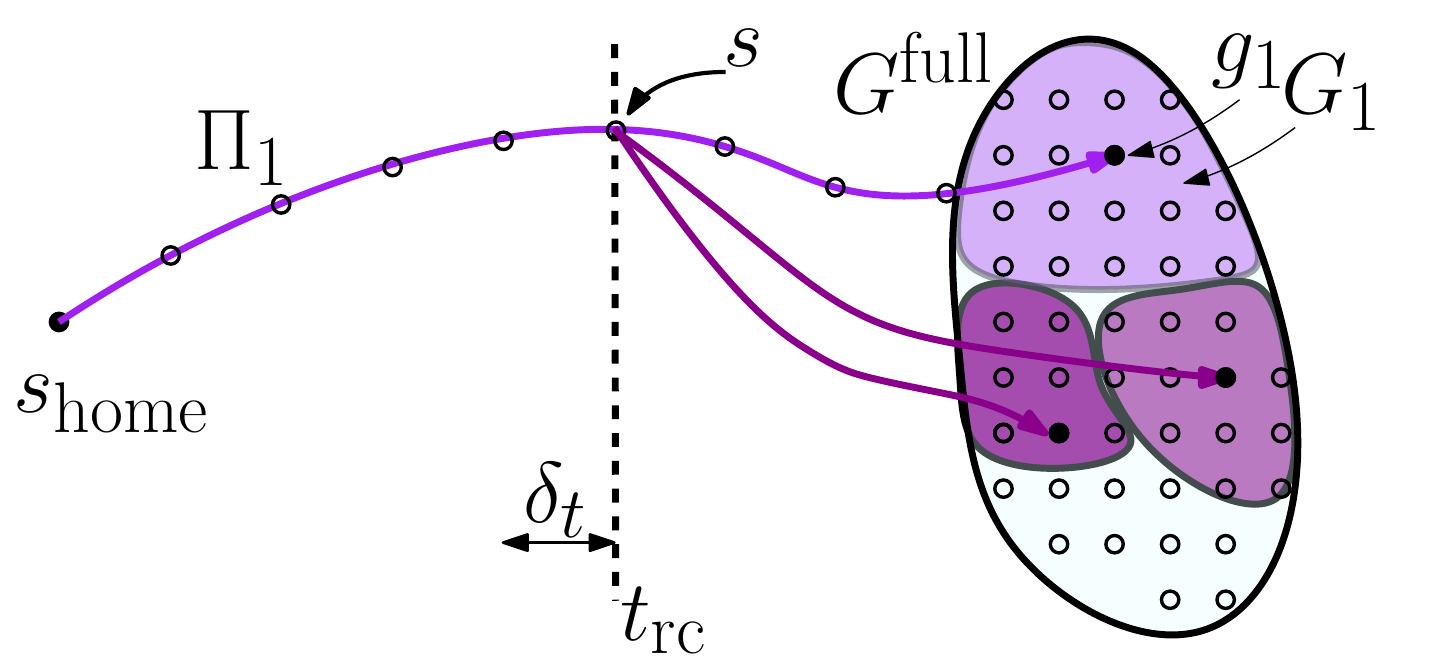}
        \caption{}
        \label{fig:pre2}
    \end{subfigure} 
    \begin{subfigure}{0.225\textwidth}
        \includegraphics[width=\textwidth]{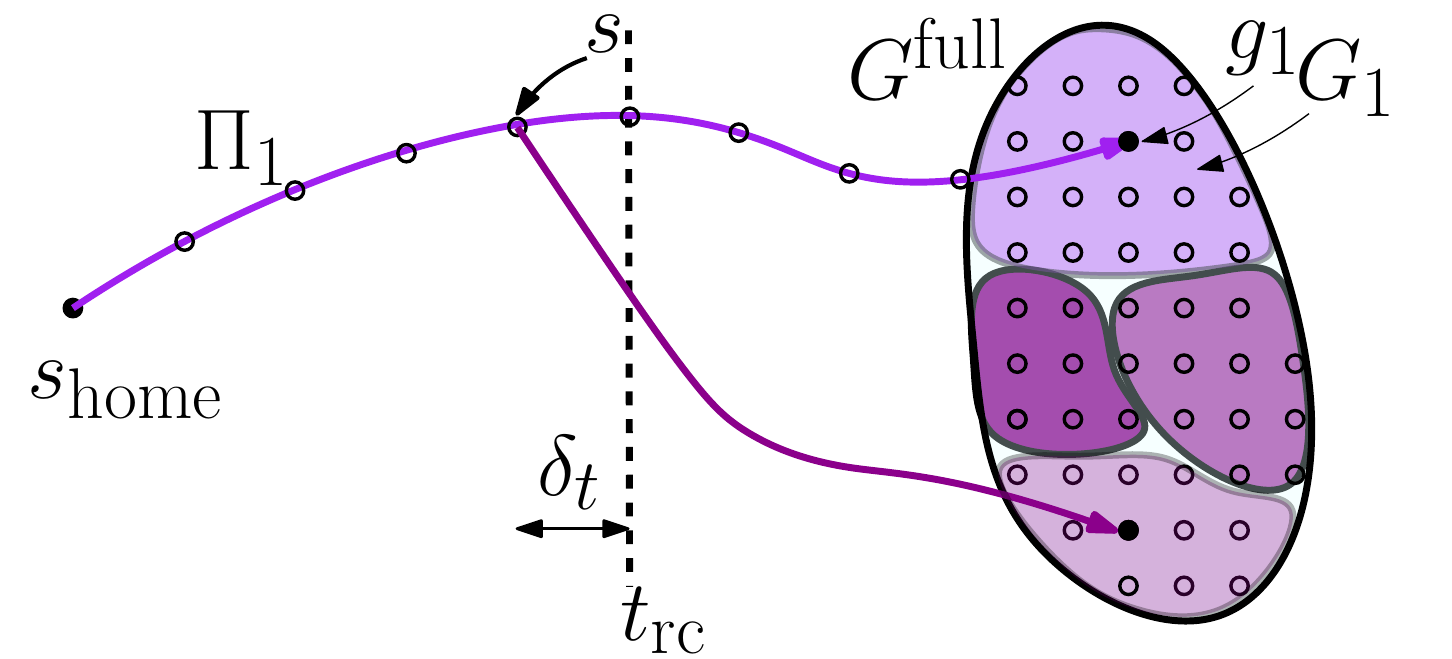}
        \caption{}
        \label{fig:pre3}
    \end{subfigure}
    \hspace{8mm}
    \begin{subfigure}{.225\textwidth}
        \includegraphics[width=\textwidth]{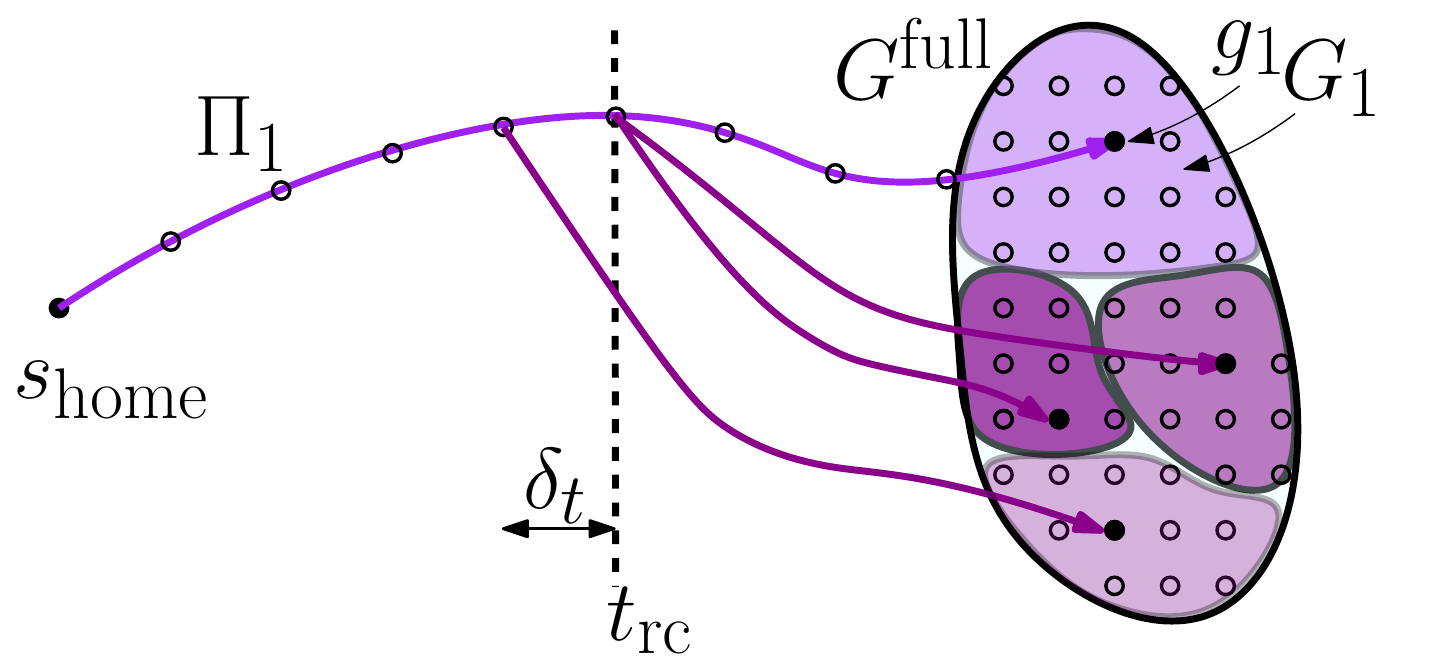}
        \caption{}
        \label{fig:pre4}
    \end{subfigure}
    \caption{\CaptionTextSize
    Preprocess loop for $\Pi_1$ without latching.
    (\subref{fig:pre1})~Initially the state $s$ covers $G_1$ via $\Pi_1$. 
    (\subref{fig:pre2})~New root paths are computed from $s$ to cover remaining uncovered region.
    (\subref{fig:pre3})~This process is repeated by backtracking along the root path.
    (\subref{fig:pre4})~Outcome of a preprocessing step for one path: \Gfull is covered either by using $\Pi_1$ as an experience or by 
    using newly-computed root paths. 
    }
    \label{fig:pl_no_latching}
\end{figure*}

\begin{figure*}[t]
    \centering
    \begin{subfigure}{.225\textwidth}
        \includegraphics[width=\textwidth]{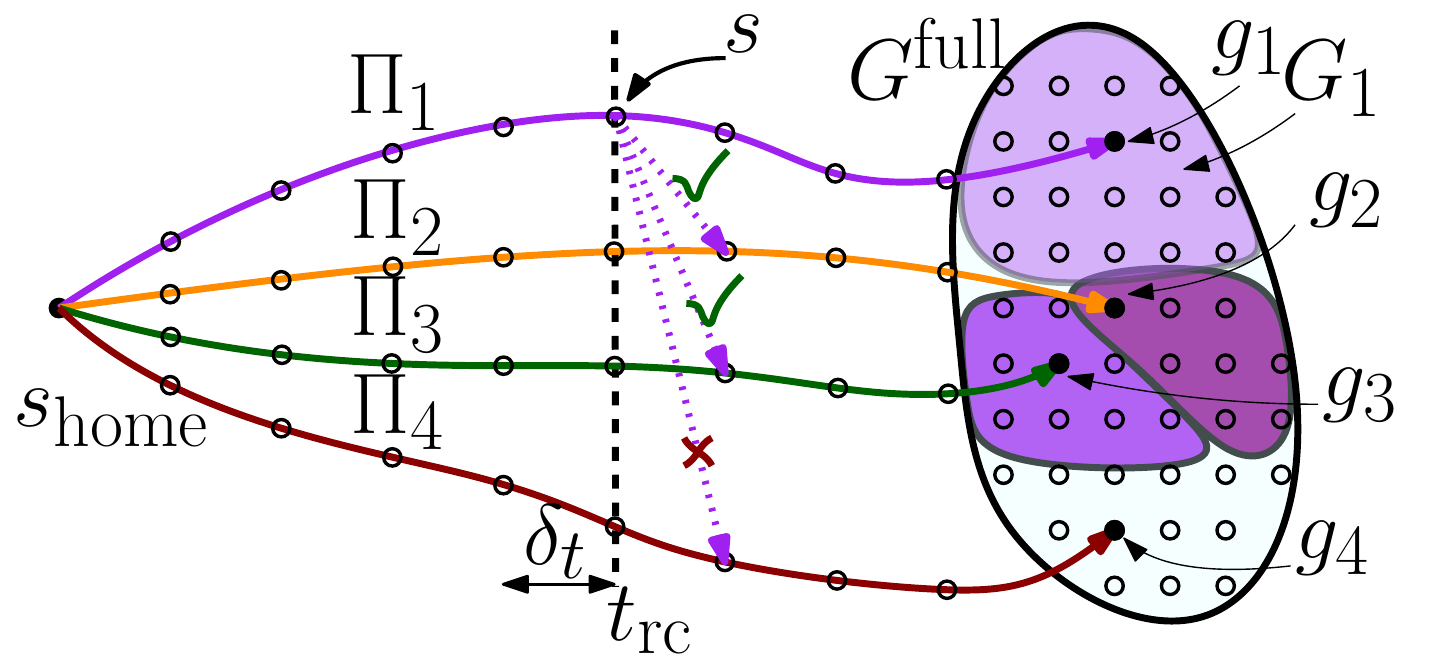}
        \caption{}
        \label{fig:pl1}
    \end{subfigure}
    \begin{subfigure}{0.225\textwidth}
        \includegraphics[width=\textwidth]{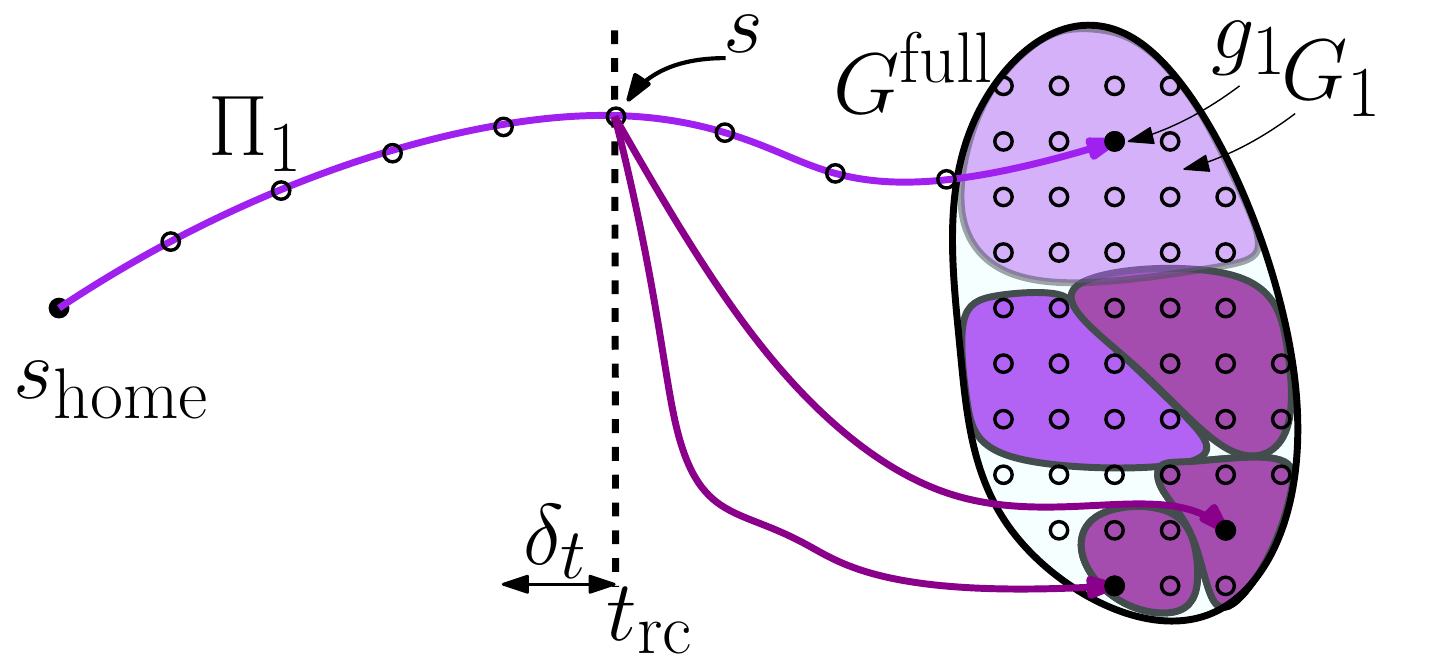}
        \caption{}
        \label{fig:pl2}
    \end{subfigure} 
    \begin{subfigure}{0.225\textwidth}
        \includegraphics[width=\textwidth]{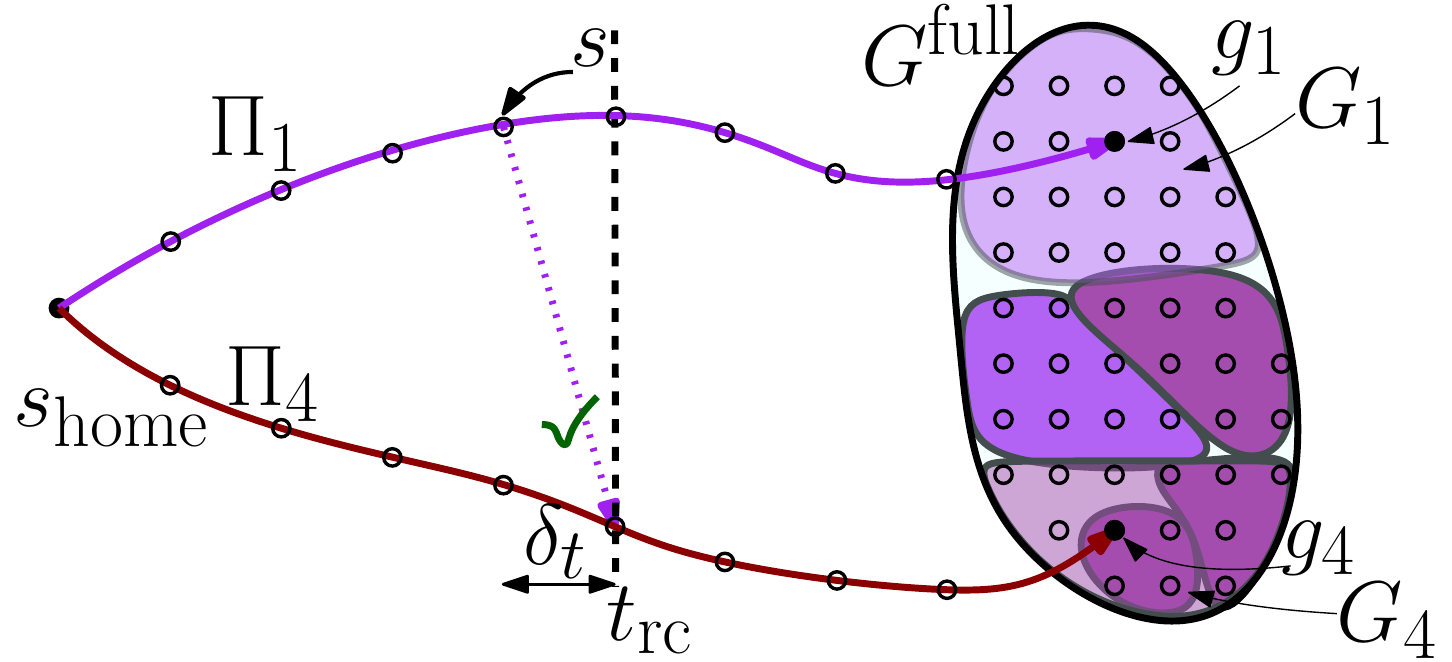}
        \caption{}
        \label{fig:pl3}
    \end{subfigure}
    \hspace{8mm}
    \begin{subfigure}{.225\textwidth}
        \includegraphics[width=\textwidth]{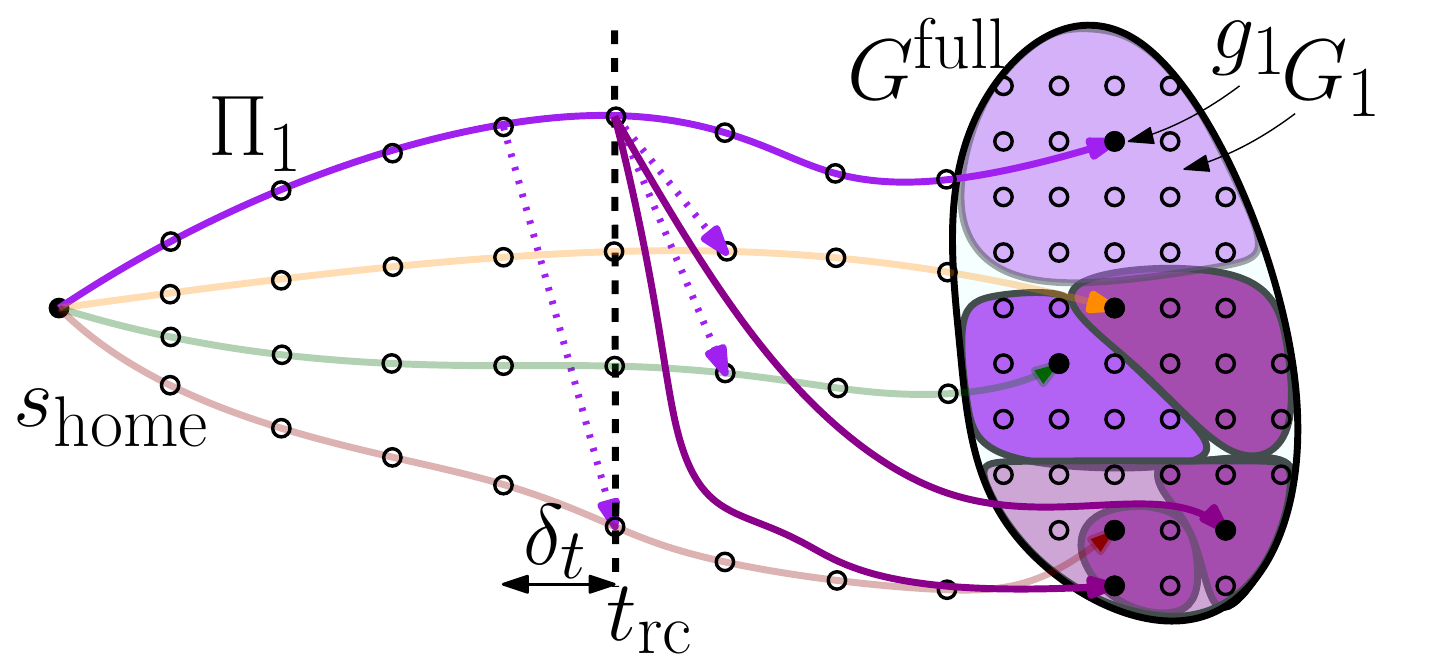}
        \caption{}
        \label{fig:pl4}
    \end{subfigure}
    \caption{\CaptionTextSize
    Preprocess loop for $\Pi_1$ with latching.
    (\subref{fig:pl1})~The algorithm starts by trying to latch on to every other root path; for successful latches, the corresponding goals are removed from uncovered region.
    (\subref{fig:pl2})~New root paths are computed from $s$ to cover remaining uncovered region.
    (\subref{fig:pl3})~This process is repeated by backtracking along the root path.
    (\subref{fig:pl4})~Outcome of a preprocessing step: \Gfull is covered either by using $\Pi_1$ as an experience, 
    latching on to $\Pi_2,\Pi_3$ or  $\Pi_4$ (at different time steps)
    or by 
    using newly-computed root paths. 
    }
    \label{fig:pl_latching}
\end{figure*}

\subsubsection{Query}
Alg.~\ref{alg:query} describes the query phase of our algorithm. Again, the lines in blue correspond to the blue pseudocode in Alg.~\ref{alg:preprocess} for the additional optimization step which is explained in Sec.~\ref{latching}.
Assume that the robot was at a state~$s_{\textrm{curr}}$ while executing a path~$\pi_{\rm curr}$ when it receives a pose update~$g$ from the perception system. Alg.~\ref{alg:query} will be called for a state \Sstart that is \Tbound ahead of~$s_{\textrm{curr}}$ along $\pi_{\rm curr}$, allowing the algorithm to return a plan before the robot reaches \Sstart.

Alg.~\ref{alg:preprocess} assures that there exists one state on~$\pi_{\rm curr}$ between \Sstart and the state at \Trc that covers $g$. Therefore, we iterate over each~$s \in \pi_{\rm curr}$ backwards (similar to Alg.~\ref{alg:preprocess}) between \Sstart and the state at \Trc and find the one that covers $g$ by quering~$\calM$. Once found, we use the corresponding root path $\Pi_{\rm next}$ as an experience to plan the path $\pi_{\rm next}$ from $s$ to $g$. Finally the paths $\pi_{\rm curr}$ and $\pi_{\rm next}$ are merged together with $s$ being the transitioning state to return the final path $\pi$.

\begin{algorithm}[t]
\caption{Query}\label{alg:query}  
    \AlgFontSize
\hspace*{\algorithmicindent} \textbf{Inputs:} $\calM, \Shome$
\begin{algorithmic}[1]
\Procedure{PlanPathByLatching}{$\Sstart,g$}
\If{$\Pi_{\textrm{home}} \leftarrow \calM (\Shome,g)$ exists} \Comment{lookup root path}
        \If{\textsc{CanLatch}($s,\Pi_{\textrm{home}}$)}
            \State $\pi_{\textrm{home}} \leftarrow$\textsc{PlanPathWithExperience}($s_{\textrm{start}},g,\Pi_{\textrm{home}}$)
            \State $\pi \leftarrow$ \textsc{MergePathsByLatching}($\pi_{\textrm{curr}},\pi_{\textrm{home}}, s$)
            \State \textbf{return} $\pi$
            \label{alg:query:line:c3b}
        \EndIf
    \EndIf
\State \textbf{return failure}
\EndProcedure
\vspace{2mm}
\Procedure{Query}{$g, \pi_{\textrm{curr}},s_{\textrm{start}}$}
    \For {\textbf{each} $s \in \pi_{\textrm{curr}}$ (from last to $\Sstart$)} \Comment{states up to $\Trc$}
        \label{alg:query:line:c2a}
         
    \If{$\Pi_{\textrm{next}} \leftarrow$ $\calM(s,g)$ exists} \Comment{lookup root path}
        \State $\pi_{\textrm{next}} \leftarrow$\textsc{PlanPathWithExperience}($s_{\textrm{start}},g,\Pi_{\textrm{next}}$)
        \State $\pi \leftarrow$ \textsc{MergePaths}($\pi_{\textrm{curr}},\pi_{\textrm{next}},s$)
        \State \textbf{return} $\pi$
        \label{alg:query:line:c2b}
    \EndIf
%
%
    \textcolor{blue}{
    \If {$\pi \leftarrow $\textsc{PlanPathByLatching}($\Sstart,g$) \textbf{successful}}
        \State \textbf{return} $\pi$
    \EndIf
    }

\EndFor
    \State \textbf{return failure} \Comment{goal is not reachable}
\EndProcedure
\end{algorithmic}
\end{algorithm}

%

\subsubsection{Latching: Reusing Root Paths}
\label{latching}
We introduce an additional step called ``Latching" to minimize the number of root paths computed in Alg.~\ref{alg:preprocess}. With latching, the algorithm tries to reuse previously-computed root paths as much as possible using special motion primitives that allow transitions from one root path to another.
The primitive is computed from a state $s \in \Pi_i$ to $s' \in \Pi_j$ such that $t(s') = t(s) + \delta_t$ by simple linear interpolation while ensuring that kinodynamic constraints of the robot are satisfied. Specifically, given the nominal joint velocities of the robot, if $s'$ can be reached from $s$ in time $\delta_t$ while respecting the kinematic and collision constraints, then the transition is allowed.

In Alg.~\ref{alg:preprocess}, before calling the \textsc{Preprocess} procedure for a state, the algorithm removes the set of goals that can be covered via latching, thereby reducing the number of goals that need to be covered by the \textsc{Preprocess} procedure. Correspondingly, in Alg.~\ref{alg:query}, an additional procedure is called to check if the path can be found via latching. These additions in the two pseudocodes are shown in blue. An iteration of the complete algorithm with latching is illustrated in Fig~\ref{fig:pl_latching}.

\subsection{Theoretical guarantees}


\begin{lemma}[Completeness]
For a robot state~$s$ and a goal~$g$, if~$g$ is \emph{reachable} from~$s$ and~$t(s) \leq t_{rc}$, the algorithm is guaranteed to find a path from~$s$ to~$g$.
\end{lemma}

\begin{proof}[Proof (Sketch)]
In order to prove it we show that (1) if~$g$ is reachable from~$s$, it is \emph{covered} by~$s$ in Alg.~\ref{alg:preprocess} and (2) if~$g$ is covered by~$s$, Alg.~\ref{alg:query} is guaranteed to return a path.

Alg.~\ref{alg:preprocess} starts by computing a set of root paths from \Shome that ensures that it covers all of its reachable goals. It then iterates over all states on these paths and adds additional root paths ensuring that these states also cover their reachable goals. It does it recursively until no state before \Trc is left with uncovered goals. Therefore, it is ensured that any reachable $g$ is covered by $s$, provided that $t(s) \leq t_{rc}$. 
    
Alg.~\ref{alg:preprocess} covers~$g$ via at least one state between~$s$ and the state at~$t_{rc}$ (inclusively) (loop at line~\ref{loop2}).
In query phase, Alg.~\ref{alg:query} iterates through all states between~$s$ and the state at~$t_{rc}$ (inclusively) to identify the one that covers~$g$ (loop at line~\ref{alg:query:line:c2a}). Since~$g$ is covered by at least one of these states by Alg.~\ref{alg:preprocess}, Alg.~\ref{alg:query} is guaranteed to find a path from $s$ to $g$.

\end{proof}

\begin{lemma}[Constant-time complexity]
\label{lemma:bounded_time}
Let $s$ be a replanable state and $g$ a goal provided by $\calP$.
If~$g$ is reachable, the planner is guaranteed to provide a solution in constant time.
\end{lemma}

\begin{proof}
    We have to show that the query stage (Alg.~\ref{alg:query}) has a constant-time complexity. 
    The number of times the algorithm queries $\calM$ which is $O(1)$ operation in case of perfect hashing is bounded by $l = \Trc/\delta_t$ which is the maximum number of time steps from $t = 0$ to $\Trc$. 
    The number of times the algorithm will attempt to latch on to a root path (namely, a call to \textsc{CanLatch}  which is a constant-time operation) is also bounded by $l$. Finally, Alg.~\ref{alg:query} calls the \textsc{Plan} method only once.
    Since the state that it is called for covers $g$, meaning that the planner can find a path from it to $g$ within \Tbound, the computation time is constant. 
    Hence the overall complexity of Alg.~\ref{alg:query} is $O(1)$.
\end{proof}

\vspace{-1mm}
\section{Evaluation}
\label{sec:eval}
We evaluated our algorithm in simulation and on a real robot. The conveyor speed that we used for all of our results is $0.2m/s$. The experiments video can be found at~\href{url}{https://youtu.be/iLVPBWxa5b8}. We used Willow Garage's PR2 robot in our experiments using its 7-DOF arm. The additional time dimension makes the planning problem eight dimensional.

\begin{table*}[t]
\centering
\begin{tabular}{|c||c||c|c|c||c|c|c||c|c|c|}
\hline
   & Our Method 
   & \multicolumn{3}{c|}{wA*}
   & \multicolumn{3}{c|}{E-Graph}
   & \multicolumn{3}{c|}{RRT}
   \\ \hline
   & $T_{b}$ = \textbf{0.2} 
   & $T_{b}$ = 0.5 & $T_{b}$ = 1.0 & $T_{b}$ = 2.0 
   & $T_{b}$ = 0.5 & $T_{b}$ = 1.0 & $T_{b}$ = 2.0 
   & $T_{b}$ = 0.5 & $T_{b}$ = 1.0 & $T_{b}$ = 2.0 
   \\ \hline
Pickup success [\%]                   
& \textbf{92.0} & 0.0 & 0.0 & 18.0 & 0.0 & 0.0 & 80.0 & 0.0 & 0.0 & 18.0 \\ \hline
Planning success [\%]                  
& \textbf{94.7} & 4.0 & 17.0 & 19.0 & 31.0 & 80.0 & 90.0 & 12.0 & 9.0 & 13.0 \\ \hline
Planning time [s]
& \textbf{0.069} & 0.433 & 0.628 & 0.824 & 0.283 & 0.419 & 0.311 & 0.279 & 0.252& 0.197\\ \hline
Planning cycles 
& \textbf{3} & 2 & 2 & 2 & 2 & 2 & 2 & 2 & 2 & 2 \\ \hline
Path cost [s]                         & 10.11        & 8.19          & 8.28          & \textbf{7.60}          & 8.54          & 8.22          & 7.90          & 9.68          & 8.96          & 8.04          \\ \hline
\end{tabular}
\caption{\CaptionTextSize Simulation results. Here $T_b$ denotes the (possibly arbitrary) timebound that the algorithm uses. Note that for our method $T_b = \Tbound$ is the time bound that the algorithm is ensured to compute a plan.}
\label{tab:sim_results}
\vspace{-2mm}
\end{table*}

\subsection{Experimental setup}

\subsubsection{Sense-plan-act cycle}
As object $o$ moves along $\calB$, we use the Brute Force ICP pose estimation baseline proposed in \cite{perch} to obtain its 3-Dof pose for each captured input point cloud.
The plan is computed for the sensed object pose projected forward by \Tbound time, giving the planner \Tbound time to plan. If the plan comes in earlier, the robot waits until the object reaches the projected pose, before executing the plan to ensure that timing is tracked properly.

%
%

\subsubsection{Goal region specification}
To define the set of all goal poses~$\Gfull$, we need to detail our system setup, depicted in Fig.~\ref{fig:pe}.
The conveyor belt $\calB$ moves along the $x$-axis from left to right.
We pick a fixed $x$-value termed~\Xexec, such that when the incoming $o$ reaches \Xexec as per the perception information, at that point we start execution.

%
Recall that a pose of an object~$o$ is a three dimensional point~$(x,y,\theta)$ corresponding to the $(x,y)$ location of~$o$ and to its orientation (yaw angle) along $\calB$.
$\Gfull$ contains a fine discretization of all possible $y$ and $\theta$ values and $x$ values in  $[\Xexec-2\varepsilon_{\calP}, \Xexec+2\varepsilon_{\calP}]$.
We select $\Gfull$ such that $\varepsilon_{\calP} = 2.5$cm, making the goal region 10cm long along x-axis. Its dimension along $y$-axis is 20cm, equal to the width of the \calB. The discretization in $x,y$ and $\theta$ is 1.0cm and 10 degrees respectively.

In the example depicted in Fig.~\ref{fig:pe}, the thick and the thin solid rectangles show the ground truth and estimated poses, respectively at two time instances in the life time of the object.
The first plan is generated for the pose shown at $x_{\textrm{exec}}$. During execution, the robot receives an improved estimate and has to replan for it. At this point we back project this new estimate in time using the known speed of the conveyor and the time duration between the two estimates. This back-projected pose (shown as the dotted rectangle) is then picked as the new goal for replanning. Recall that under the assumption~\ref{assum:3} the back projected pose will always lie inside \Gfull.
%

\begin{figure}
\vspace{2mm}
    \centering
    \includegraphics[width=0.49\textwidth]{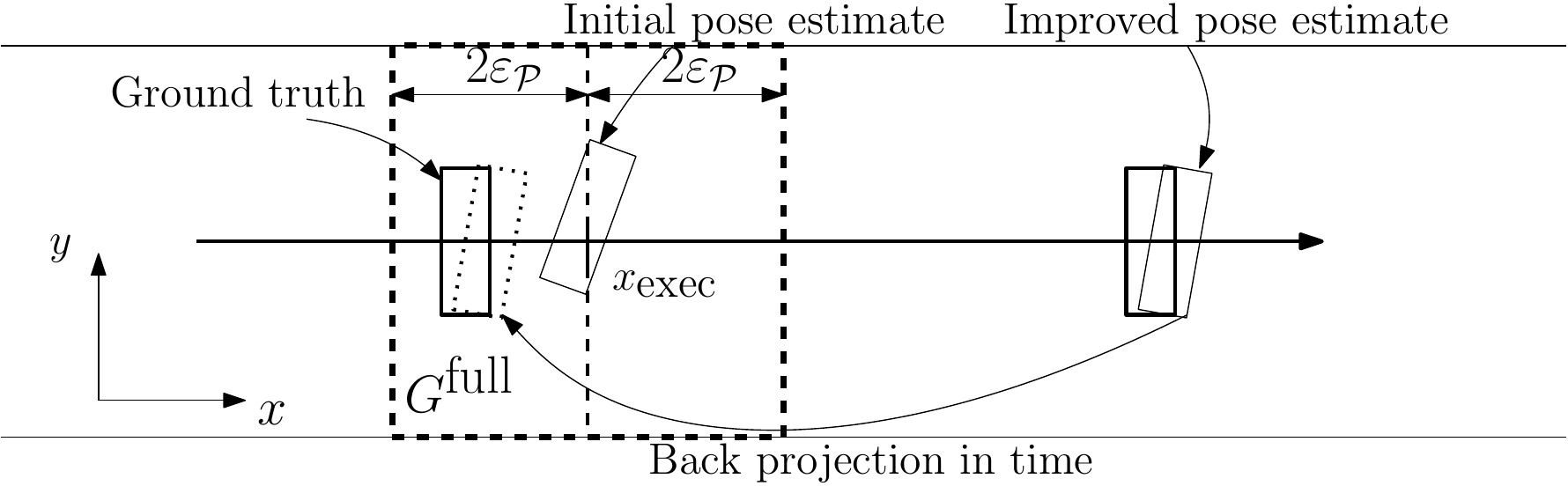}
    \caption{\CaptionTextSize A depiction of \Gfull-specification on a conveyor belt (overhead view) and perception noise handling. 
    }
    \label{fig:pe}
        \vspace{-4mm}
\end{figure}

\subsection{Results}

The preprocessing stage (i.e. running Alg.~\ref{alg:preprocess}) took roughly 3.5 hours and the memory footprint following this stage was less than 20Mb.
This supports our intuition that the domain allows for efficient compression of a massive amount of paths in a reasonable amount of preprocessing time. In all experiments, we used $\Trc =3.5$s and $\delta_t = 0.5$s.

\subsubsection{Real-robot experiments}
\label{sec:robot_results}
To show the necessity of real-time replanning in response to perception updates, we performed three types of experiments, 
(E1)~using our approach to replan every time new object pose estimate arrives, 
(E2)~single-shot planning based on the first object pose estimate 
(E3)~single-shot planning using the late (more accurate) pose estimate. 
For each set of experiments, we determined the pickup success rate to grasp the moving object (sugar box) off~$\calB$. In addition, we report on $\calP$'s success rate by observing the overlap between the point cloud of the object's 3D model transformed by the predicted pose (that was used for planning) and the filtered input point cloud containing points belonging to the object. 
A high (low) overlap corresponds in an accurate (inaccurate) pose estimate. 
We use the same strategy to determine the range for which $\calP$'s estimates are accurate and use it to determine the time for the best-pose planning. 
Further, for each method, we determine the pickup success rate given that the \calP's estimate was or wasn't accurate. 

The experimental results are shown in Table \ref{tab:robot_results}. Our method achieves the highest overall pickup success rate on the robot by a large margin, indicating the importance of continuous replanning with multiple pose estimates. 
First-pose planning has the least overall success rate due to inaccuracy of pose estimates when the object is far from the robot's camera. 
Best-pose planning performs better overall than the first pose strategy, since it uses accurate pose estimates, received when the object is close to the robot. However it often fails even when perception is accurate, since a large number of goals are unreachable due to limited time remaining to grasp the object when it is closer to the robot.

\begin{table}[t]
\centering
\begin{tabular}{|c|c|c|c|c|}
\hline
                    & \begin{tabular}[c]{@{}c@{}}Success \\ rate\end{tabular}
                    & \begin{tabular}[c]{@{}c@{}}Accuracy of \\ \calP $[\%]$ \end{tabular} 
                    & \begin{tabular}[c]{@{}c@{}}Success rate \\ (Accurate \calP)\end{tabular} 
                    & \begin{tabular}[c]{@{}c@{}}Success rate \\ (Inaccurate \calP)\end{tabular} \\ \hline
E1          & \textbf{69.23}                                                                      & 42.31                   & \textbf{83.33}                                                                             & \textbf{57.14}                                                                             \\ \hline
E2 & 16.00                                                                      & 24.00                   & 66.67                                                                             & 0.00                                                                              \\ \hline
E3   & 34.61                                                                      & 34.62                   & 55.56                                                                             & 23.53                                                                             \\ \hline
\end{tabular}
\caption{\CaptionTextSize Real-robot experiments. Success rate for the three experiments (E1---our method, E2---First-pose planning and E3---Best-pose planning).}
\label{tab:robot_results}
\vspace{-4mm}
\end{table}

\ignore{
\begin{table*}[t]
\centering
\begin{tabular}{|c|c|c|c|c|}
\hline
                    & \begin{tabular}[c]{@{}c@{}}Pickup success rate [\%]\\ (Overall)\end{tabular} & Perception success rate [\%]& \begin{tabular}[c]{@{}c@{}}Pickup success rate [\%]\\ (Perception = 1*)\end{tabular} & \begin{tabular}[c]{@{}c@{}}Pickup success rate [\%]\\ (Perception = 0*)\end{tabular} \\ \hline
Our method (E1)          & \textbf{9.23}                                                                      & \textbf{42.31}                   & \textbf{83.33}                                                                             & \textbf{57.14}                                                                             \\ \hline
First-pose planning (E2) & 16.00                                                                      & 24.00                   & 66.67                                                                             & 0.00                                                                              \\ \hline
Best-pose planing (E3)   & 34.61                                                                      & 34.62                   & 55.56                                                                             & 23.53                                                                             \\ \hline
\end{tabular}
\caption{Real-robot Experiments}
\label{tab:robot_results}
\end{table*}
}

\subsubsection{Simulation experiments}

We simulated the real world scenario to evaluate our method against other baselines. We compared our method with wA*~\cite{pohl1970heuristic}, E-graph~\cite{PCCL12} and RRT~\cite{lavalle1998rapidly}. 
For wA* and E-graph we use the same graph representation as our method. 
For E-graph we precompute five paths to randomly-selected  goals in \Gfull. 
We adapt the RRT algorithm to account for the under-defined goals. To do so, we sample pre-grasp poses along the conveyor 
and compute IK solutions for them to get a set of goal configurations for goal biasing. 
When a newly-added node falls within a threshold distance from the object, we use the same dynamic primitive that we use in the search-based methods to add the final grasping maneuver. If the primitive succeeds, we return success. We also allow wait actions at the pre-grasp locations.

For any planner to be used in our system, we need to endow it with a (possibly arbitrary) planning time bound to compute the future location of the object from which the new execution will start.
%
If the planner fails to generate the plan within this time, then the robot misses the object for that cycle and such cases are recorded as failures. 
We label a run as a success if the planner successfully replans once after the object crosses the 1.0m mark. The mark is the mean of accurate perception range that was determined experimentally and used in the robot experiments as described in Section \ref{sec:robot_results}.
The key takeaway from our experiments (Table~\ref{tab:sim_results}) is that having a known time bound on the query time is vital to the success of the conveyor pickup task.

Our method shows the highest pickup success rate, planning success rate (success rate over all planning queries) and an order of magnitude lower planning times compared to the other methods. 
The planning success rate being lower than 100\% can be attributed to the fact that some goals were unreachable during the runs. 
We tested the other methods with several different time bounds. After our approach E-graph performed decently well. RRT suffers from the fact that the goal is under-defined and sampling based planners typically require a goal bias in the configuration space. Another important highlight of the experiments is the number of planning cycles over the lifetime of an object. While the other approaches could replan at most twice, our method was able to replan thrice due to fast planning times.

\vspace{-3mm}
\section{Conclusion}
To summarize, we developed a provably constant-time planning and replanning algorithm that can be used to grasp fast moving objects off conveyor belts and evaluated it in simulation and in the real world on the PR2 robot. Through this work, we advocate the need for algorithms that guarantee (small) constant-time planning for time critical applications, such as the conveyor pickup task, which are often encountered in warehouse and manufacturing environments.


\ignore{
\begin{figure}
    \centering
    \includegraphics[width=0.4\textwidth]{figs/preprocess.png}
    \caption{Preprocessing}
    \label{fig:preprocessing}
\end{figure}
}




\balance
\bibliographystyle{unsrt}
\bibliography{references}

\end{document}

%% file: macros.tex
\newcommand{\calG}{\ensuremath{\mathcal{G}}\xspace}
\newcommand{\calR}{\ensuremath{\mathcal{R}}\xspace}
\newcommand{\calM}{\ensuremath{\mathcal{M}}\xspace}

\newcommand{\calB}{\ensuremath{\mathcal{B}}\xspace}

\newcommand{\calP}{\ensuremath{\mathcal{P}}\xspace}

\newcommand{\calO}{\ensuremath{\mathcal{O}}\xspace}

\def\naive{{na\"{\i}ve}\xspace}

\newtheorem{thm}{Theorem}

\newtheorem{definition}[thm]{Definition}

\newcommand{\ignore}[1]{}

\newcommand\Gfull{\ensuremath{G^{\textrm{full}}}\xspace}
\newcommand\Gcov{\ensuremath{G^{\textrm{cov}}}\xspace}
\newcommand\Guncov{\ensuremath{G^{\textrm{uncov}}}\xspace}

\newcommand\Guncovp{\ensuremath{G'^{\textrm{uncov}}}\xspace}
\newcommand\Shome{\ensuremath{s_{\textrm{home}}}\xspace}
\newcommand\Tbound{\ensuremath{T_{\textrm{bound}}}\xspace}
\newcommand\Trc{\ensuremath{t_{\textrm{rc}}}\xspace}
\newcommand\Ssc{\ensuremath{s_{\textrm{sc}}}\xspace}
\newcommand\Sstart{\ensuremath{s_{\textrm{start}}}\xspace}
\newcommand\Xexec{\ensuremath{x_{\textrm{exec}}}\xspace}

\def\os#1{\textcolor{magenta}{#1}}

\def\AlgFontSize{\small}
\def\CaptionTextSize{\small}